\documentclass{article} %
\usepackage{iclr2024_conference, times}

\usepackage{titletoc}

\usepackage{amsmath,amsfonts,bm}

\def\eqref#1{equation~\ref{#1}}
\def\Eqref#1{Equation~\ref{#1}}

\def\1{\bm{1}}

\def\vone{{\bm{1}}}
\def\vmu{{\bm{\mu}}}
\def\vtheta{{\bm{\theta}}}
\def\vomega{{\bm{\omega}}}
\def\vdelta{{\bm{\delta}}}
\def\vone{{\bm{1}}}
\def\va{{\bm{a}}}
\def\vb{{\bm{b}}}

\def\vu{{\bm{u}}}

\def\vx{{\bm{x}}}

\def\vz{{\bm{z}}}

\DeclareMathAlphabet{\mathsfit}{\encodingdefault}{\sfdefault}{m}{sl}
\SetMathAlphabet{\mathsfit}{bold}{\encodingdefault}{\sfdefault}{bx}{n}

\newcommand{\vlambda}{\boldsymbol{\lambda}}

\newcommand{\vpi}{\boldsymbol{\pi}}

\newcommand{\second}{\cellcolor{blue!10}}
\newcommand{\first}{\cellcolor{blue!30}}

\definecolor{cyan}{cmyk}{.3,0,0,0}
\definecolor{lightblue}{HTML}{B0C4DE}
\definecolor{darkblue}{HTML}{177CB0}

\usepackage[pagebackref=true,unicode=true,colorlinks=true,citecolor=blue,linkcolor=red]{hyperref}
\usepackage{url}
\usepackage{lipsum}
\usepackage{graphicx} %
\usepackage{amsmath} %
\usepackage{amsfonts} %
\usepackage{amsthm} %
\usepackage{amssymb} %
\usepackage{bbold} %
\usepackage[only,llbracket,rrbracket]{stmaryrd}
\usepackage{amsthm} %
\usepackage{booktabs}
\usepackage{multirow}
\usepackage{caption}
\usepackage{subcaption}
\usepackage{textcomp} %

\usepackage{xcolor}
\usepackage{colortbl}

\usepackage{comment}
\title{A Symmetry-Aware Exploration of Bayesian Neural Network Posteriors}
\iclrfinalcopy
\author{Olivier Laurent,\textsuperscript{\rm 1,2} Emanuel Aldea,\textsuperscript{\rm 1} \& \textbf{Gianni Franchi}\textsuperscript{\rm 2, $\dagger$} \\
\AND
{
\normalfont SATIE, Université Paris-Saclay,\textsuperscript{\rm 1} U2IS, ENSTA Paris, Institut Polytechnique de Paris,\textsuperscript{\rm 2}}
}

\begin{NoHyper}
\footnotetext{corresponding author - gianni.franchi@ensta-paris.fr}
\end{NoHyper}

\newtheorem*{theorem*}{Theorem}
\newtheorem{proposition}{Proposition}
\newtheorem*{proposition*}{Proposition}
\newtheorem{corollary}{Corollary}
\newtheorem*{corollary*}{Corollary}
\newtheorem{lemma}{Lemma}
\newtheorem{property}{Property}[section]
\newtheorem*{property*}{Property}
\theoremstyle{definition}
\newtheorem{definition}{Definition}[section]
\newtheorem*{definition*}{Definition}
\theoremstyle{remark}

\begin{document}

\maketitle

\begin{abstract}
 The distribution of the weights of modern deep neural networks (DNNs) - crucial for uncertainty quantification and robustness - is an eminently complex object due to its extremely high dimensionality. This paper proposes one of the first large-scale explorations of the posterior distribution of deep Bayesian Neural Networks (BNNs), expanding %
 its study to real-world vision tasks and architectures. Specifically, we investigate the optimal approach for approximating the posterior, analyze the connection between posterior quality and uncertainty quantification, delve into the impact of modes on the posterior, and explore methods for visualizing the posterior. Moreover, we uncover weight-space symmetries as a critical aspect for understanding the posterior. To this extent, we develop an in-depth assessment of the impact of both permutation and scaling symmetries that tend to obfuscate the Bayesian posterior. While the first type of transformation is known for duplicating modes, we explore the relationship between the latter and L2 regularization, challenging previous misconceptions. Finally, to help the community improve our understanding of the Bayesian posterior, we %
 will release shortly the first large-scale checkpoint dataset, including thousands of real-world models, along with our codes.

\end{abstract}

\begin{figure}[!thb]
    \centering
    \includegraphics[width=0.84\linewidth]{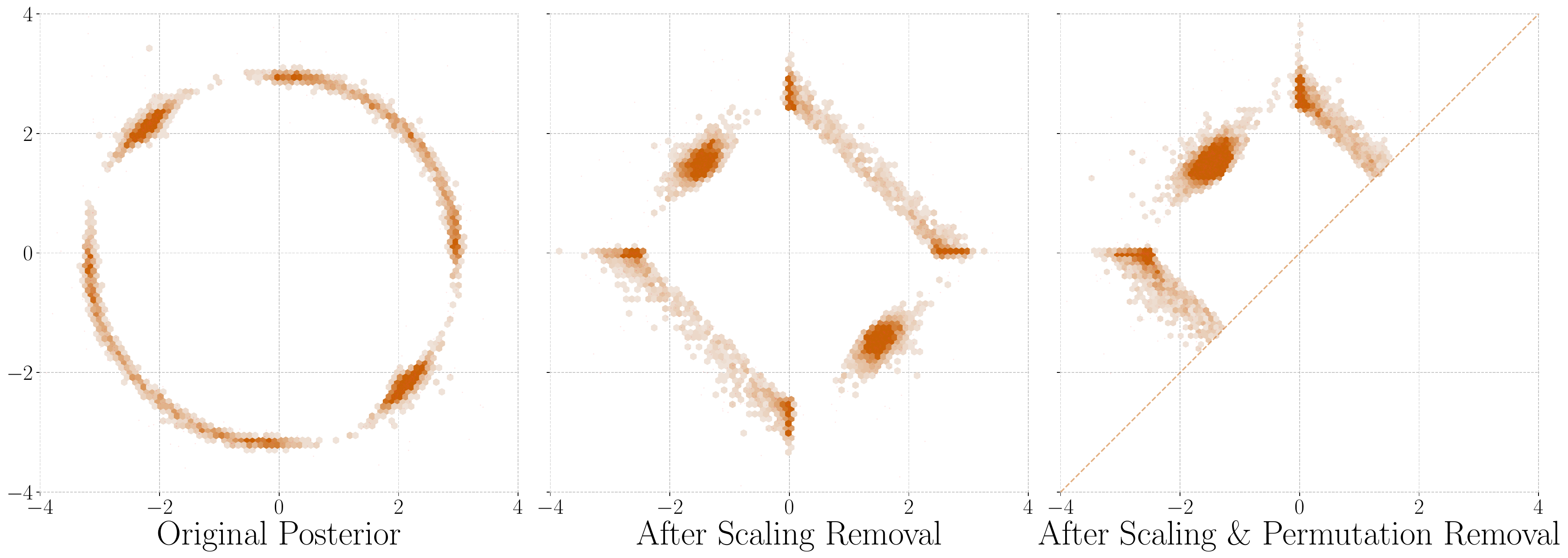}
    \caption{\textbf{Weight-space symmetries greatly impact the estimated Bayesian posterior}. Permutation symmetries very clearly increase the number of modes of the posterior distribution in the case of a 2-hidden neuron perceptron.}
    \label{fig:example}
\end{figure}

\section{Introduction}

Despite substantial advancements in deep learning, Deep Neural Networks (DNNs) %
remain black box models. Various studies have sought to explore DNN loss landscapes~\citep{li2018visualizing,fort2019large,fort2019goldilocks,liu2022loss} to achieve a deeper understanding of these models. %
Recent works have, for instance, unveiled the interconnection of the modes obtained with Stochastic Gradient Descent (SGD) via narrow pathways that link pairs of modes, or through tunnels that connect multiple modes simultaneously~\citep{garipov2018loss, draxler2018essentially}.
This mode connectivity primarily arises from %
scale and permutation invariances, which imply that numerous weights can represent the same exact function (e.g., \cite{entezari2021role}). 
Several studies have delved into the relationship between these symmetries and the characteristics of the loss landscape~\citep{entezari2021role,neyshabur2015path,brea2019weight}. Our work investigates the connections between these symmetries and the distribution of DNN weights, a crucial aspect for uncertainty quantification. As shown in Figure~\ref{fig:example}, it is %
apparent that these symmetries also exert influence on DNN posterior distributions.

Uncertainty quantification plays a pivotal role in high-stakes industrial applications - such as autonomous driving~\citep{levinson2011towards, mcallister2017concrete, sun2019functional} - where reliable predictions and informed decision-making are paramount. In such critical domains, understanding and effectively managing uncertainties, particularly the model-related epistemic uncertainties~\citep{hora1996aleatory} arising from incomplete knowledge, is essential. Among the various methods proposed to address these challenges, Bayesian Neural Networks (BNNs)~\citep{tishby1989consistent} offer a principled and theoretically sound approach. BNNs quantify uncertainty by modeling our beliefs about parameters and outcomes probabilistically~\citep{tishby1989consistent,hinton1993keeping}. However, this perspective faces significant hurdles when applied to deep learning, primarily related to scalability~\citep{izmailov2021bayesian} and the precision of approximations~\citep{mackay1995probable}. Due to their very high dimension, BNNs struggle to estimate the posterior distribution, i.e., the probability density of converging towards any particular set of model parameters $\omega$, given the observed data $\mathcal{D}$.

Diverging from methods such as the Maximum Likelihood Estimate or Maximum A Posteriori (also in \cite{tishby1989consistent}), which we typically derive through gradient descent optimization of cross-entropy (with L2 regularization for the latter), BNNs assign a probability to each possible model (or hypothesis) and offer predictions considering the full extent of possible models. In mathematical terms, when we denote the target as $y$, the input vector as $\vx$, and
the weight space as $\Omega$, we can express this approach through the following intractable formula, often referred to as the marginalization on the parameters of the model~\citep{tishby1989consistent, rasmussen2006gaussian}:
\begin{align}
\label{eq:marginalization}
    p(y \mid \vx, \mathcal{D})=\int\limits_{\vomega\in\Omega} p(y \mid \vx, \vomega) p(\vomega \mid \mathcal{D}) d\vomega.
\end{align}
The posterior distribution $p(\vomega|\mathcal{D})$ assumes a central and arguably the most critical role in BNNs. Many successful methods for quantifying uncertainty can be viewed as attempts to approximate this posterior, each with its own trade-offs in terms of accuracy and computational efficiency, as illustrated in previous research~\citep{blundell2015weight, gal2016dropout, simple2017lakshminarayanan}. Prior work~\citep{kuncheva2003measures, fort2019deep,ortega2022diversity} has established the importance of achieving \emph{diversity} in the sampled DNNs drawn from the posterior, particularly when dealing with uncertain input data. However, the presence of permutation symmetries and scale symmetries among hidden units in neural networks may lead to an increased number of local minima~\citep{zhao2022symmetries} with no diversity. In the context of BNNs, this phenomenon can result in a proliferation of modes within the posterior distribution.

In this paper, we delve into the impact of weight symmetries on the posterior distribution. While there have been numerous efforts to visualize the loss landscape, we explore the possibility of conducting similar investigations for the posterior distribution. Additionally, we introduce a protocol for assessing the quality of posterior estimation and examine the relationship between posterior estimation and the accuracy of uncertainty quantification. Specifically, our contributions  are as follows:

\textbf{(1)} We build a new mathematical formalism to highlight the different impacts of the permutation and scaling symmetries on the posterior and on uncertainty estimation in deep neural networks. Notably, we explain the seeming equivalence of the marginals in Figure~\ref{fig:example}. We also perform the first in-depth exploration of the existence of scaling symmetries and their overlooked effect. \textbf{(2)}~We evaluate the quality of various methods for estimating the posterior distribution on real-world applications using the Maximum Mean Discrepancy, offering a practical benchmark to assess their performance in capturing uncertainty. \textbf{(3)}~We release a new dataset that includes the weights of thousands of models across various computer vision tasks and architectures, ranging from MNIST to TinyImageNet. This dataset is intended to facilitate further exploration and collaboration in the field of uncertainty in deep learning. \textbf{(4)}~Our investigation delves into the proliferation of modes in the context of posterior symmetries and exhibits the capacity of ensembles to converge toward non-functionally equivalent modes. Furthermore, we discuss the influence of symmetries in the training process.

\section{Related work}

\paragraph{Epistemic uncertainty, Bayesian inference, and posterior}

Epistemic uncertainty~\citep{hora1996aleatory,hullermeier2021aleatoric} plays a crucial role in accurately assessing predictive model reliability. However, despite ongoing discussions, estimating this uncertainty remains a challenge. %
BNNs~\citep{goan2020bayesian} predominantly shape the landscape of methodologies that tackle epistemic uncertainties~\citep{gawlikowski2023survey}. Given the complexity of dealing with posterior distributions, these approaches have mostly been tailored for enhanced scalability.

For instance, \cite{hernandez2015probabilistic} proposed an efficient probabilistic backpropagation, and \cite{blundell2015weight} developed BNNs by Backpropagation to learn diagonal Gaussian distributions with the reparametrization trick. Similarly, Laplace methods~\citep{laplace1774memoire, mackay1992practical, ritter2018scalable} characterize the posterior distribution of weights, often focusing on the final layer~\citep{ober2019benchmarking,watson2021latent}, again for scalability.

On a different approach, Monte Carlo Dropout, introduced by \cite{gal2016dropout} and \cite{kingma2015variational}, is a framework that models the posterior as a mixture of Dirac distributions when applied on fully-connected layers. Broadening the spectrum, Deep Ensembles~\citep{simple2017lakshminarayanan}, arguably along with their more efficient alternatives~\citep{lee2015m, wen2019batchensemble, maddox2019simple, franchi2020encoding, franchi2020tradi, havasi2021training, laurent2023packed}, have been interpreted by \cite{wilson2020bayesian} as Monte Carlo estimates of \Eqref{eq:marginalization}.

\paragraph{Markov-chain-based Bayesian posterior estimation}

\cite{neal2011mcmc} introduced HMC as an accurate method for estimating the posterior of neural networks, but its application to large-scale problems remains challenging due to the very high computational demands. \cite{izmailov2021bayesian} managed to scale full-batch HMC to CIFAR-10~\citep{krizhevsky2009learning} with ResNet-20~\citep{he2016deep} thanks to 512 TPUv3.

In response to these challenges, stochastic approximations of MCMC have gained attention for their ability to provide computationally more feasible solutions. A prominent example is the Langevin dynamics-based approach proposed by \cite{welling2011bayesian}. By introducing noise into the dynamics, Langevin dynamics allows for more practical implementation on large datasets. %

In addition to Langevin dynamics, other stochastic gradient-based methods have been introduced to improve the efficiency of MCMC sampling. \cite{chen2014stochastic} presented Stochastic Gradient Hamiltonian Monte Carlo (SGHMC), integrating stochastic gradients into HMC. Likewise, \cite{zhang2019cyclical} proposed C-SGLD and C-SGHMC (Cyclic Stochastic Gradient Langevin Dynamics and Hamiltonian Monte Carlo), introducing controlled noise via cyclic preconditioning.

While stochastic approximation methods offer computational convenience, they come with the trade-off of slowing down the convergence and potentially introducing bias into the resulting inference~\citep{bardenet2017markov,zou2021convergence}. As such, the suitability of these approaches depends on the specific application and the level of acceptable bias in the analysis.

\paragraph{Symmetries in neural networks}

The seminal work from \cite{hecht1990algebraic} established a foundational understanding by investigating permutation symmetries and setting a lower bound on symmetries in multi-layer perceptrons. \cite{albertini1993uniqueness} extended this work and studied flip sign symmetries in networks with odd activation functions. This work was further generalized to a broader range of activation functions by \cite{kuurkova1994functionally}, who proposed symmetry removal to streamline evolutionary algorithms.

Recent advancements have generalized symmetries to modern neural architectures. \cite{neyshabur2015path} explored the scaling symmetries that arise in architectures containing non-negative homogeneous activation functions, including \cite{nair2010rectified}'s ubiquitous Rectified Linear Unit (ReLU). This perspective extends our understanding of symmetries to ReLU-powered architectures, e.g., AlexNet~\citep{krizhevsky2012imagenet}, VGG~\citep{simonyan2014very}, and ResNet networks~\citep{he2016deep}. This paper focuses on scaling and permutation symmetries, but other works, such as \cite{rolnick2020reverse, grigsby2023hidden}, unveil less apparent symmetries. 

Closest to our work, \cite{wiese2023towards} demonstrated that taking weight-space symmetries into account could reduce the Bayesian posterior support and improve Monte-Carlo Markov Chains posterior estimation.

\section{Symmetries increase the complexity of Bayesian posteriors}

We propose to study the most influential weight-space symmetries and their properties related to the Bayesian posterior. Firstly, weight-space symmetries transform the parameters of the neural networks while keeping the networks functionally invariant; in other words,

\begin{definition}
Let $f_\vomega$ be a neural network of parameters $\vomega$ taking $n$-dimensional vectors as inputs.
We say that the transformation $\mathcal{T}$ modifying $\vomega$ is a weight-space symmetry operator, iff 
\begin{equation}
    f_{\mathcal{T}(\vomega)} = f_\vomega, \ \text{or} \ \  \forall \vx \in\mathbb{R}^n, f_{\mathcal{T}(\vomega)}(\vx) = f_\vomega(\vx).
\end{equation}
\end{definition}

With the notation $f_{\mathcal{T}(\vomega)}(\vx)$, we apply the symmetry operator $\mathcal{T}$ on %
the weights $\vomega$, resulting in a set of modified weights. In this paper, we show that scaling and permutation symmetries have different impacts on the posterior of neural networks. They can, for instance, complicate Bayesian posteriors, creating artificial functionally equivalent modes.

\subsection{A gentle example of artificial symmetry-based posterior modes}
\label{sec:ex_details_mp}
To illustrate the considerable impact of symmetries on the Bayesian posterior, we propose a small-scale example in Figure~\ref{fig:example}. We generate this example by training two-hidden-neuron perceptrons on linearly separable data. %
Figure~\ref{fig:example} presents the estimation of the posterior of the output weights with 10,000 checkpoints. In this figure, the most important mode is duplicated due to the permutation symmetry, which appears to symmetrize the graph. On the other hand, scaling symmetries seem to regroup some points around the modes. We detail this toy experiment in Appendix~\ref{sec:ex_details}. In the following, we develop a new mathematical framework tailored to help understand the impact of these symmetries on the posterior, devise mathematical insights explaining these intuitions, and propose empirical explorations.

\subsection{Background and definitions}
The full extent of this formalism (including sketches of proofs, other definitions, properties, and propositions) is developed in Appendix~\ref{app:formalism}. We deem that the following properties are the minimal information necessary to understand the impact on the Bayesian posterior of the two main types of symmetries: the scaling and permutation symmetries. %
This part summarizes the most important results for multi-layer perceptrons, but we provide leads for generalizing our results to modern deep neural networks such as convolutional residual networks~\citep{he2016deep} in Appendix~\ref{sec:generalization_leads}.%

\subsection{Scaling symmetries}
For clarity, we propose the following definitions and properties for two-layer fully connected perceptrons without loss of generality. Please refer to Appendix~\ref{app:formalism} and~\ref{sec:generalization_leads} for extensions. We start with two new operators between vectors and matrices to define scaling symmetries.

\begin{definition}
We denote the line-wise product as $\bigtriangledown$ and the column-wise product as $\rhd$. Let $\vomega \in \mathbb{R}^{n\times m}$, $\vlambda \in \mathbb{R}^m$, and $\vmu\in \mathbb{R}^n$, with $n, m \geq 1$, 
\begin{equation}
    \forall i\in\llbracket 1, n\rrbracket, \forall j\in\llbracket 1,m\rrbracket, \ (\vlambda \bigtriangledown \vomega)_{i,j} = \lambda_j\omega_{i,j} \ \text{and} \ (\vmu \rhd \vomega)_{i,j} = \mu_i\omega_{i,j}.
\end{equation}
\end{definition}

Given that the rectified linear unit $r$ is non-negative homogeneous - i.e., for all non-negative $\lambda$, $r(\lambda x) = \lambda r(x)$ - we have the following core property for scaling symmetries~\citep{neyshabur2015path}, trivially extendable to additive biases:

\begin{property}
For all $\vtheta\in\mathbb{R}^{\cdot\times m}$, $\vomega \in\mathbb{R}^{m\times n}$, $\vlambda \in \left(\mathbb{R}_{>0}\right)^m$,
\begin{equation}
\label{eq:scaling_twolay_def}
    \forall\vx\in\mathbb{R}^n, \ (\vlambda^{-1}\bigtriangledown \vtheta) \times r(\vlambda\rhd\vomega\vx) = \vtheta \times r(\vomega\vx).
\end{equation}
\end{property}

When defining $\mathcal{T}_s$ by the transformation \Eqref{eq:scaling_twolay_def} - in the case of a two-layer perceptron - the core property directly follows, with the set of parameters $\Lambda = \{\vlambda\}$:

\begin{property}[Scaling symmetries]
The scaling operation $\mathcal{T}_s$ with a set of non-negative parameters $\Lambda$ is a symmetry for all neural network $f_\vomega$, i.e.
\begin{equation}
    \forall\vx\in\mathbb{R}^n, \ f_{\mathcal{T}_s(\vomega, \Lambda)}(\vx) = f_\vomega(\vx).
\end{equation}
\end{property}

\subsection{Permutation symmetries}
We also propose a simple and intuitive formalism for permutation symmetries, multiplying the weights by permutation matrices. We have the following property for two-layer perceptions, with $P_m$ the set of $m$-dimensional permutation matrices:

\begin{property}
For all $\vtheta\in\mathbb{R}^{\cdot\times m}$, $\vomega \in\mathbb{R}^{m\times n}$, and permutation matrices $\vpi \in P_m$,
\begin{equation}
\label{eq:perm_twolay_def}
     \forall\vx\in\mathbb{R}^n, \ \vtheta \vpi^\intercal \times r\left(\vpi \times \vomega\vx\right) = \vtheta \times r(\vomega\vx).
\end{equation}
\end{property}

The left term of \Eqref{eq:perm_twolay_def} is the definition of the permutation symmetry operator of parameter $\Pi=\{\vpi\}$ for a network with two layers. In general, we have the following property:

\begin{property}[Permutation symmetries]
The permutation operation $\mathcal{T}_p$ with a set of parameters $\Pi$ is a symmetry for all neural networks $f_\vomega$, i.e.,
\begin{equation}
    \forall\vx\in\mathbb{R}^n, \ f_{\mathcal{T}_p(\vomega, \Pi)}(\vx) = f_\vomega(\vx).
\end{equation}
\end{property}

\subsection{The Bayesian posterior as a mixture of distributions}
\label{sec:thm}

With this formalism, we can establish the following proposition, clarifying the impact of weight-space symmetries on the Bayesian posterior.

\begin{proposition}
\label{prop:mixture}
Define $f_\omega$ a neural network and $f_{\tilde{\omega}}$ its corresponding identifiable model - a network transformed for having sorted unit-normed neurons. Let us also denote $\mathbb{\Pi}$ and $\mathbb{\Lambda}$, the sets of permutation sets and scaling sets, respectively, and $\tilde{\Omega}$ the random variable of the sorted weights with unit norm. The Bayesian posterior of a neural network $f_\omega$ trained with stochastic gradient descent can be expressed as a continuous mixture of a discrete mixture:
 \begin{equation}
 \label{eq:mixture}
         p(\Omega=\vomega \mid \mathcal{D}) = \hspace{-0.3em}
     \int\limits_{\Lambda \in \mathbb{\Lambda}} \hspace{-0.3em} |\mathbb{\Pi}|^{-1} \hspace{-0.2em}\sum\limits_{\Pi\in\mathbb{\Pi}}p(\tilde{\Omega} = \mathcal{T}_p(\mathcal{T}_s(\vomega,\Lambda), \Pi), \Lambda \mid \mathcal{D})d\Lambda.
\end{equation}
\end{proposition}

Proposition~\ref{prop:mixture} provides an expression of the Bayesian posterior that highlights the redundancy of the resulting distribution, explaining the symmetry in Figure~\ref{fig:example} (left). Interestingly, a direct corollary of this formula is that \emph{all the layer-wise marginal distributions are identical}. In \Eqref{eq:mixture}, the permutations play a transparent role, being independent of $\vomega$ (except for strongly quantized spaces for $\vomega$, but we leave this particular case for future works). On the other hand, the part played by scaling symmetries is more complex and we discuss their impact in the following section.

\subsection{On the effective impact of scaling symmetries}
\label{sec:minimum_mass_pbm}

The equiprobability of permutations in \Eqref{eq:mixture} leads to a balanced mixture of $|\mathbb{\Pi}|$ permuted terms. We have no such result on scaling symmetries since the standard L2-regularized loss is not invariant to scaling symmetries (contrary to permutations). This absence of invariance obscures the impacts of scaling symmetries, which mostly remains to be addressed, although their "reduction" due to regularization was mentioned in, e.g., \cite{godfrey2022symmetries}. To the best of our knowledge, we provide the first analysis of the reality of scaling symmetries and their impact on the Bayesian posterior. With this objective in mind, we propose the following problem.

\begin{definition}
    Let $f_{\omega}$ be a neural network and $\bar{\omega}$ its weights without the biases. We define the minimum scaled network mass problem (or \emph{the min-mass problem}) as the minimization of the L2-regularization term of $f_{\omega}$ (the "mass") under scaling transformations. In other words,
    \begin{equation}
        m^* = \min\limits_{\Lambda\in\mathbb{\Lambda}}\ \left|\mathcal{T}_s(\bar{\omega}, \Lambda)\right|^2.
    \end{equation}
\end{definition}

This problem has interesting properties. Notably, we show in Appendix~\ref{app:loglogconvex} that the \emph{min-mass} problem is log-log strictly convex~\citep{boyd2004convex,agrawal2019disciplined}:

\begin{proposition}
\label{prop:loglogconvexity}
    The minimum scaled network mass problem is log-log strictly convex. As such, it is equivalent to a strictly convex problem on $\mathbb{R}^{|\mathbb{\Lambda}|}$ and admits a global minimum attained in a single point denoted $\Lambda^*$.
\end{proposition}

\begin{figure}[t]
    \centering
    \begin{subfigure}[b]{0.33\textwidth}
         \centering
         \includegraphics[width=\textwidth]{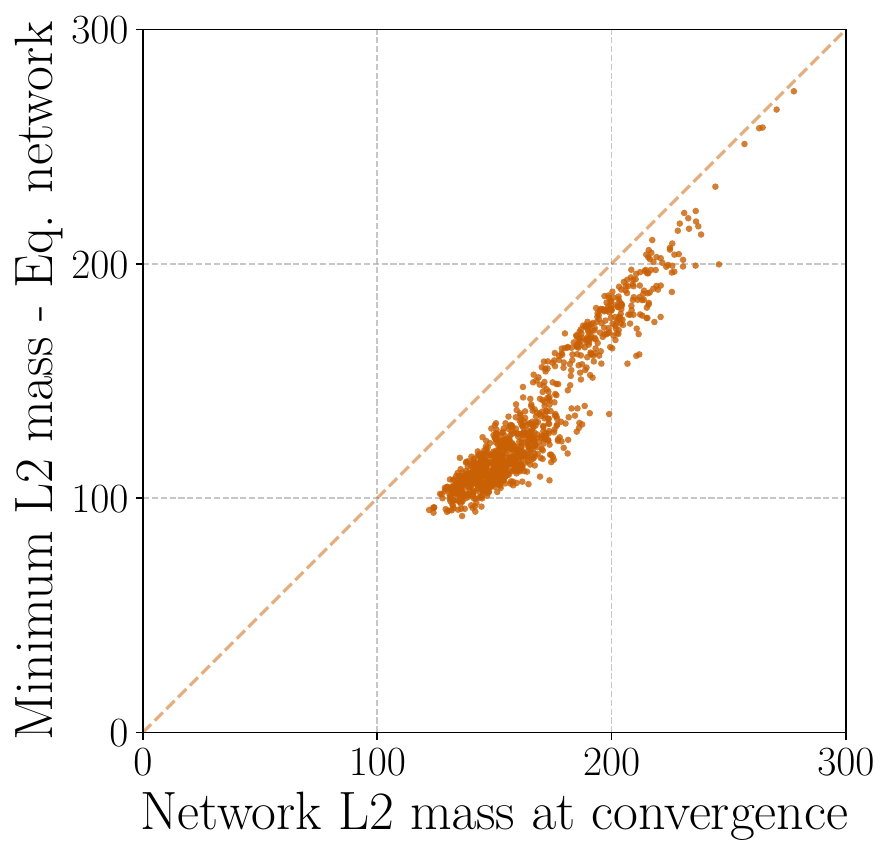}
     \end{subfigure}
     \hfill
     \begin{subfigure}[b]{0.32\textwidth}
         \centering
         \includegraphics[width=\textwidth]{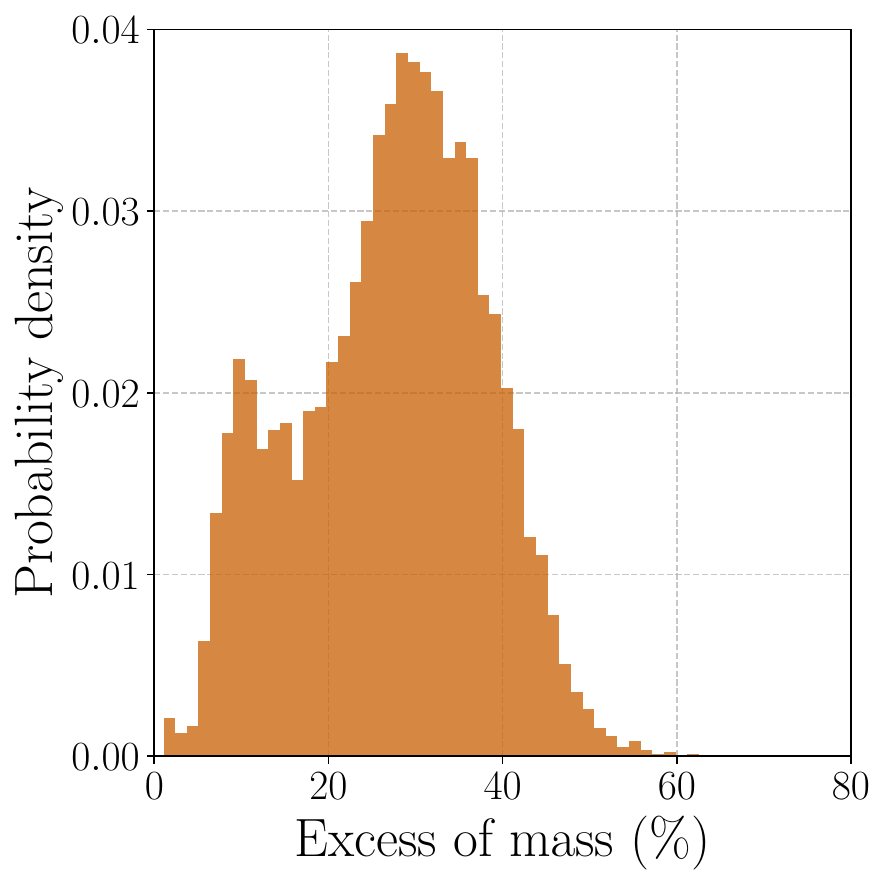}
     \end{subfigure}
     \hfill
     \begin{subfigure}[b]{0.31\textwidth}
         \centering
         \includegraphics[width=\textwidth]{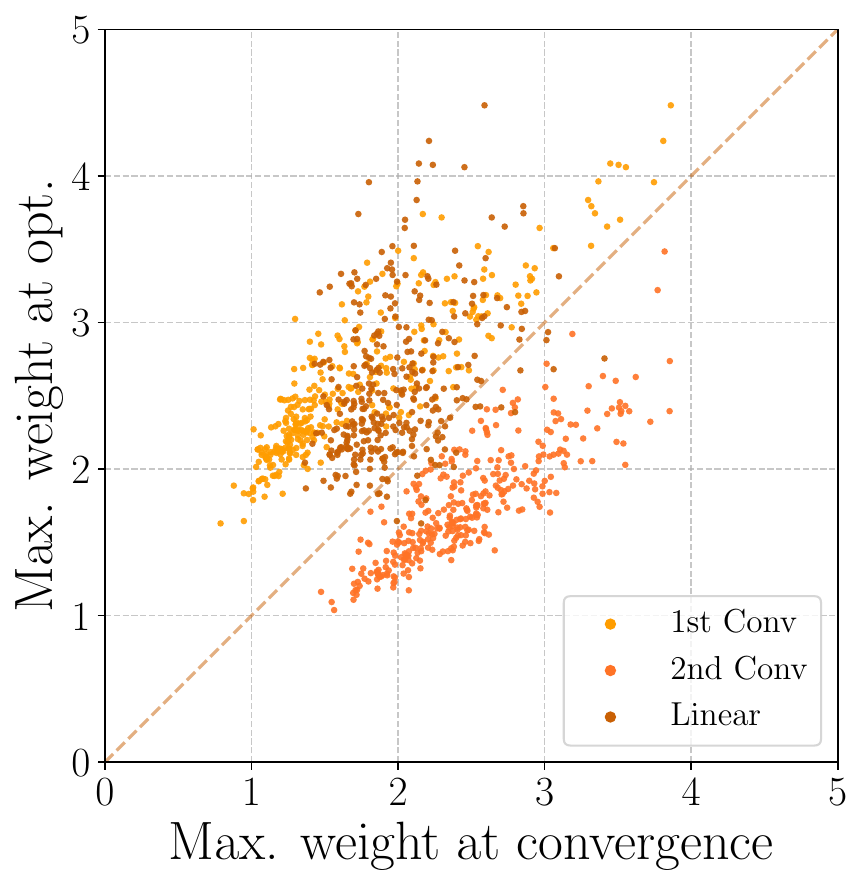}
     \end{subfigure}
        \caption{\textbf{OptuNet trained with weight decay never converges towards the \emph{minimum mass}.} We note that the maxima of the weights of scaled OptuNets at the minimum mass tend to be greater for lighter layers than in the original networks.}
        \label{fig:optunet_masses}
\end{figure}

It follows from Proposition~\ref{prop:loglogconvexity} that, if not already optimal at convergence, there is an infinite number of equivalent networks with training loss lower than the original network. We put the proposition into practice in Figure~\ref{fig:optunet_masses} using trained OptuNets (see Appendix~\ref{sup:arch_detail}). We measure their mass distribution and compare it to the masses of the optimal networks found with convex optimization. In practice, the effect of scaling symmetries remains present even with weight decay: neural networks always seem subject to scaling symmetries. For completeness, we report the values of the largest elements of each layer, showing that the minimization tends to increase the maxima of the layers with fewer parameters yet does not seem to promote unfeasible values.

In this section, we provided theoretical insights on the contrasted impacts of both scaling and permutation symmetries. In the following, we perform more empirical studies and propose to evaluate the link between posterior quality and experimental performance.

\section{Comparing the estimations of the posterior by approximate Bayesian methods}
\label{sec:comparison}

\DeclareRobustCommand{\perthousand}{%
  \ifmmode
    \text{\textperthousand}%
  \else
    \textperthousand
  \fi}
  
\begin{table}[t]
\centering
\resizebox{0.98\textwidth}{!}{%
\begin{tabular}{@{}cccccccccccc@{}}
\toprule[1.4pt]
 &  & Method & MMD $\downarrow$ & NS $\downarrow$ & Acc $\uparrow$ & ECE $\downarrow$ & Brier $\downarrow$ & AUPR $\uparrow$  & FPR95 $\downarrow$ & \textbf{ID}MI $\downarrow$ & \textbf{OOD}MI $\uparrow$ \\ \midrule
\multirow{10}{*}{\rotatebox[origin=c]{90}{MNIST - OptuNet}} & \multirow{5}{*}{\rotatebox[origin=c]{90}{One Mode}} & 
      Dropout   &  15.0 & 14.3 & 83.3 & 33.4 & 26.1 & 96.4 & 98.6 & 26.1 & 22.2 \\
 &  & BNN   & 18.8 & 17.1 & 78.1 & \first \textbf{7.4} & 30.9 & 67.9 & 93.7 & \first \textbf{0.1} & 0.1 \\
  &  & SGHMC  & 16.7 & 17.7 & \second 95.1 & \second 7.6 & \first \textbf{2.8} & 73.7 & 98.4 & 4.3 & 14.5 \\
 &  & SWAG  & 16.0 & 14.6 & 88.3 & 17.7 & 4.9 & 73.4 & 68.6 & \second 4.0 & 8.7 \\
 &  & Laplace  & 10.6 & 9.5 & 87.9 & 18.1 & \second 4.8 & 48.2 & 74.6 & 6.2 & 5.9 \\ \cmidrule[0.5pt](l){2-12} 
 & \multirow{5}{*}{\rotatebox[origin=c]{90}{Multi Mode}} & 
      Dropout   & 2.1 & 2.1 & 92.1 & 29.2 & 36.8 & \first \textbf{97.2} & 78.2 & 36.6 & 52.5 \\
 &  & BNN  & 2.8 & 2.5 & 86.5 & 17.5 & 24.4 & \second 96.9 & 27.2 & 21.1 & 52.3 \\
 &  & SWAG  & \second 1.8 & 1.3 & 95.0 & 13.1 & 17.5 & 88.7 & \second 24.6 & 27.6 & \second 62.2 \\
 &  & Laplace  & \second 1.8 & \second 0.8 & 94.8 & 12.8 & 15.8 & 95.4 & 32.1 & 21.1 & 52.2 \\
 &  & DE  & \first \textbf{0.0} & \first \textbf{0.0} & \first \textbf{95.3} & 10.7 & 13.5 & 95.7 & \first \textbf{12.8} & 19.3 & \first \textbf{62.6} \\ \midrule 
 \multirow{10}{*}{\rotatebox[origin=c]{90}{CIFAR100 - ResNet18}} & \multirow{5}{*}{\rotatebox[origin=c]{90}{One Mode}} & 
      Dropout  & 4.5 & 7.5 & 69.3 & 12.1 & 44.3 & 80.1 & 64.0 & 6.0 & 9.6 \\
 &  & BNN  & 9.0 & 10.2 & 57.6 & 21.8 & 62.5 & 81.7 & 62.6 & \second 0.8 & 2.2 \\
  &  & SGHMC  & 7.5 & 7.9 & 69.3 & 4.3 & 41.5 & 87.5 & 41.4 & \first \textbf{0.0} & 0.1 \\
 &  & SWAG  & 6.7 & 7.2 & 66.7 & 1.7 & 43.8 & 30.7 & 17.1 & 2.7 & 14.7 \\
 &  & Laplace  & 5.7 & 7.0 & 70.7 & \second 1.4  &40.0 & 84.5 & 40.6 & 31.3 & 70.6 \\  \cmidrule[0.5pt](l){2-12} 
 & \multirow{5}{*}{\rotatebox[origin=c]{90}{Multi Mode}} & 
      Dropout  & 0.7 & 4.5 & \second 75.5 & 4.8 & 64.2 & 93.5 & 23.7 & 23.8 & 77.0 \\
 &  & BNN & 6.1 & 5.6 & 66.3 & \first \textbf{1.3} & 45.6 & 91.0 & 30.9 & 44.5 & \second 110.4 \\
 &  & SWAG  & 5.0 & 5.4 & 68.4 & 2.3 & 42.1 & \first \textbf{97.9} & \second 17.1 & 7.5 & 33.0 \\
 &  & Laplace  & \second 0.6 & \second 4.3 & 75.2 & 7.8 & \second 35.2 & 95.1 & 20.1 & 50.5 & \first \textbf{128.3} \\
 &  & DE  & \first \textbf{0.0} & \first \textbf{0.0} & \first \textbf{75.9} & 2.3 & \first \textbf{33.3} & \second 97.1 & \first \textbf{16.1} & 26.5 & 90.1 \\ \midrule %
 \multirow{10}{*}{\rotatebox[origin=c]{90}{TinyImageNet - ResNet18}} & \multirow{5}{*}{\rotatebox[origin=c]{90}{One Mode}} & 
      Dropout  & 9.5 & 4.9 & 56.4 & 14.2 & 60.3 & 70.8 & 87.6 & 9.5 & 13.5 \\
 &  & BNN  & / & / & / & / & / & / & / & / & / \\
  &  & SGHMC  & 9.8 & 5.3 & 54.4 & \first \textbf{2.3} & 58.9 & 78.1 & 74.4 & \first \textbf{0.1} & 0.2 \\
 &  & SWAG  & 9.1 & 3.9 & 61.7 & 10.0 & 51.8 & 84.6 & 66.7 & \second 3.3 & 9.0 \\
 &  & Laplace  & 5.5 & 6.1 & 29.7 & 7.4 & 81.3 & 62.3 & 82.3 & 211.7 & \second 254.8 \\ \cmidrule[0.5pt](l){2-12} 
 & \multirow{5}{*}{\rotatebox[origin=c]{90}{Multi Mode}} & 
      Dropout  & 4.3 & 1.8 & \second 65.3 & 10.5 & 48.2 & 90.9 & 57.7 & 41.9 & 93.2 \\
 &  & BNN & / & / & / & / & / & / & / & / & / \\
 &  & SWAG  & 6.7 & 5.4 & 64.8 & \first \textbf{2.3} & \first \textbf{46.6} & \first \textbf{98.4} & \second 53.6 & 20.5 & 55.2 \\
 &  & Laplace  & \second 0.5 &  \second 3.1 & 34.4 & 13.0 & 79.3 & 64.4 & 77.2 & 227.4 & \first \textbf{267.2}  \\
 &  & DE  & \first \textbf{0.0} & \first \textbf{0.0} & \first \textbf{65.4} & 8.5 & \second 47.3 & \second 94.2 & \first \textbf{41.6} & 44.7 & 108.5 \\ \bottomrule[1.4pt]
\end{tabular}%
}
\caption{\textbf{Comparison of popular methods approximating the Bayesian posterior.} All scores are expressed in \%, except the MMDs for ResNet-18 networks, expressed in \perthousand. Acc stands for accuracy, and \textbf{ID}MI and \textbf{OOD}MI are the in-distribution and out-of-distribution mutual information. NS is the MMD computed after the removal of the symmetries, and DE stands for Deep Ensembles. Multi-mode methods use ten independently trained models.}
\label{tab:main-table}
\end{table}

In this section, we leverage symmetries to compare popular single-mode methods, namely, Monte Carlo Dropout~\citep{gal2016dropout}, Stochastic Weight Averaging Gaussian (SWAG) proposed by \cite{maddox2019simple}, Bayes by backpropagation BNNs~\citep{blundell2015weight}, and Laplace methods~\citep{ritter2018scalable}.
We also include their multi-modal variations, corresponding to the application of these methods on 10 different independently trained models, as well as SGHMC~\citep{chen2014stochastic} and Deep Ensembles (DE) highlighted by \cite{simple2017lakshminarayanan} and proposed earlier by \cite{hansen1990neural}. We compare these methods on 3 image classification tasks with different levels of difficulty, ranging from MNIST~\citep{lecun1998gradient} with our OptuNet (392 parameters) to CIFAR-100~\citep{krizhevsky2009learning} and Tiny-ImageNet~\citep{deng2009imagenet} with ResNet-18~\citep{he2016deep}. To this extent, we propose leveraging maximum mean discrepancy (MMD) to estimate the dissimilarities between the high-dimensional posterior distributions. We estimate the target distribution with 1000 checkpoints and compare it to 100 samples extracted from each of the previously-mentioned techniques. We detail these experiments in Appendix~\ref{sec:exp_details}.

\subsection{Evaluating the quality of the estimation of the Bayesian posterior}

One approach to assess the similarities between distributions involves estimating the distributions and subsequently quantifying the distance between these estimated distributions~\citep{smola2007hilbert,sriperumbudur2010hilbert}. However, these methods can become impractical when dealing with distributions in extremely high-dimensional spaces, such as the posterior of DNNs. An alternative solution is to embed the probability measures into a Reproducing Kernel Hilbert Space (RKHS) developed by, e.g., \cite{bergmann1922entwicklung, aronszajn1950theory, schwartz1964sous}. Within this framework, a distance metric MMD~\citep{song2008learning} - defined as the distance between the respective mean elements within the RKHS - is used to quantify the dissimilarity between the distributions.
\cite{gretton2012kernel} proposed to leverage MMD for efficient two-sample tests in high dimensions. For better efficacy, MMDs are computed on RKHS that allow for comparing all of their moments. In Table~\ref{tab:main-table}, we propose to compute the median of the aggregated MMDs~\citep{schrab2021mmd} with multiple Gaussian and Laplace kernels, with and without symmetries.

For tractability, we report the mean - weighted by the number of parameters of each layer - of the median over the twenty MMD kernels between the layer-wise estimated posterior and the posterior approximated by each method. The estimated posterior includes 1000 independently trained checkpoints, and the posterior approximation of each method includes 100 samples. Please refer to Appendix~\ref{sec:exp_details} for extended details (including the means and maxima of the MMDs).

\subsection{Performance metrics and OOD datasets}

On top of the MMD quantifying the difference between the posterior estimations, we measure several empirical performance metrics. We evaluate the overall performance of the models using the accuracy and the Brier score~\citep{brier1950verification, gneiting2007probabilistic}. Furthermore, we choose the binned Expected Calibration Error (ECE)~\citep{naeini2015obtaining} for top-label calibration and measure the quality of the out-of-distribution (OOD) detection using the area under the precision-recall curve (AUPR) and the false positive rate at $95\%$ recall (FPR95), as recommended by \cite{hendrycks2016baseline}, as a measure of OOD detection abilities, that are supposed to correlate with the quality of the estimated posterior. Finally, we report the mean diversity of the predictions in each ensemble through the mutual information (MI) (e.g., \cite{ash1965information}), often used to measure epistemic uncertainty~\citep{kendall2017uncertainties, ovadia2019can}. We use FashionMNIST~\citep{xiao2017fashion}, SVHN~\citep{netzer2011reading}, and Textures~\citep{cimpoi14describing} as OOD datasets for MNIST, CIFAR-100, and TinyImageNet, respectively.

\subsection{Results}

In our examination of Table \ref{tab:main-table} emerges that multi-mode techniques consistently demonstrate superior performance in terms of MMD when compared to their single-mode counterparts. This trend holds true in accuracy and negative log-likelihood. However, an intriguing divergence arises when considering the ECE, where techniques focusing solely on estimating a single mode often exhibit superior performance.

Turning our attention to the assessment of epistemic uncertainty, as quantified by AUPR and FPR95, multi-mode techniques, notably multi-SWAG, and Deep Ensembles, consistently outperform other methods. This underscores the strong connection between posterior estimation and the accuracy of epistemic uncertainty quantification. However, we note that the quality of aleatoric uncertainty quantification does not consistently correlate with that of the posterior distribution estimation.

The final two columns of the table shed light on the diversity of the models sampled from the posterior. The objective here is to minimize in-distribution mutual information (IDMI) while concurrently maximizing out-of-distribution mutual information (OODMI). An analysis shows that mono-mode methods tend to yield lower values for both IDMI and OODMI than multi-mode methods, which tend to exhibit higher IDMI and OODMI values, suggesting greater \emph{diversity}.

\section{Discussions}

We develop further insights on the posterior of Bayesian Neural Networks in relationship with symmetries. Notably, we evaluate the risk of \emph{functional collapse}, i.e., of training very similar networks in Section \ref{disc:functional_collapse} and we discuss the frequency of weights permutations in \ref{disc:perm}. We expand these discussions, add observations on the evaluation of the number of modes of the posterior, and propose visualizations in Appendix~\ref{sec:disc_details}.

\subsection{Functional collapse in ensembles: a study of ID and OOD disagreements}
\label{disc:functional_collapse}
\begin{figure}[t]
    \centering
    \begin{subfigure}[b]{0.32\textwidth}
         \centering
         \includegraphics[width=\textwidth]{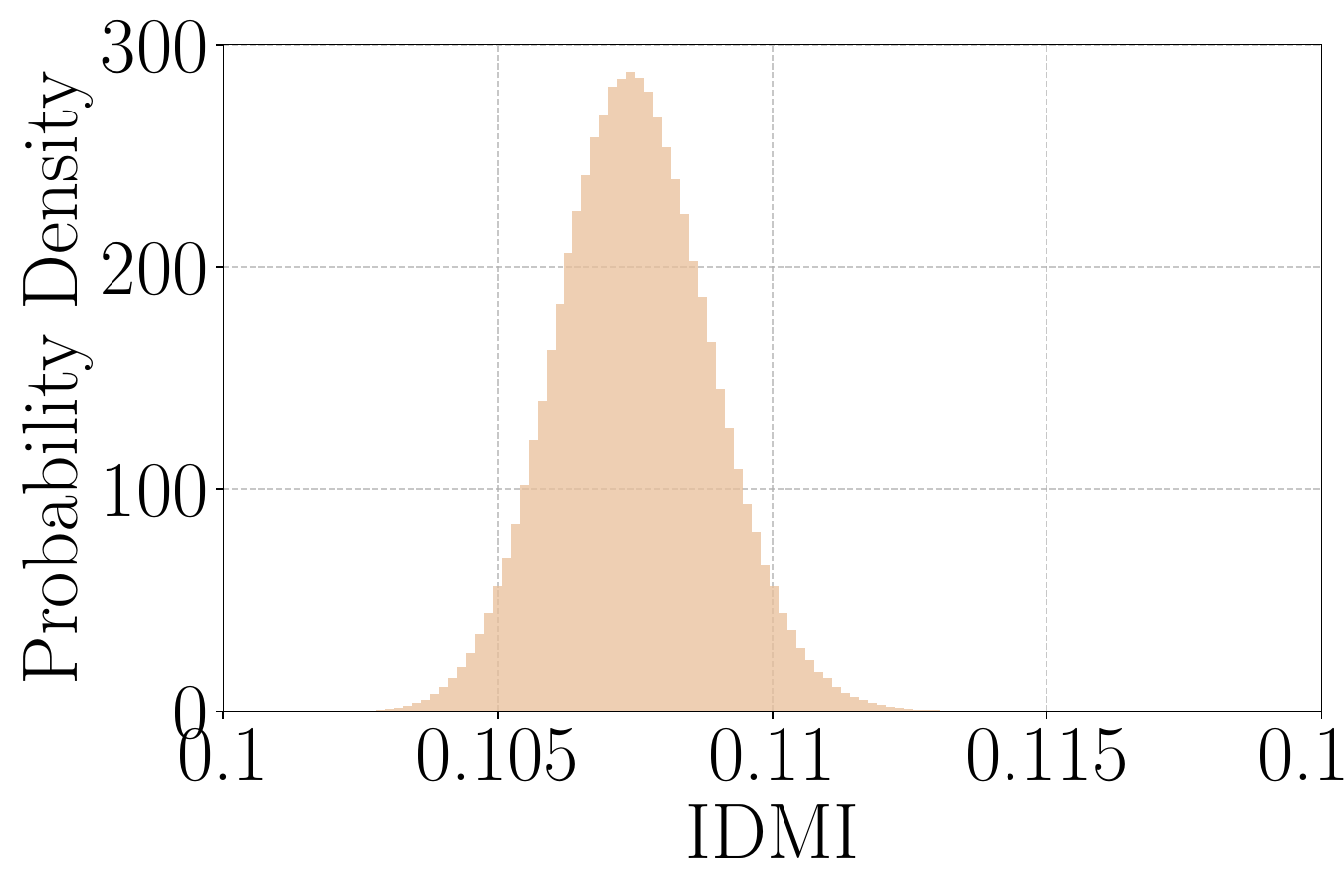}
     \end{subfigure}
     \hfill
     \begin{subfigure}[b]{0.32\textwidth}
         \centering
         \includegraphics[width=\textwidth]{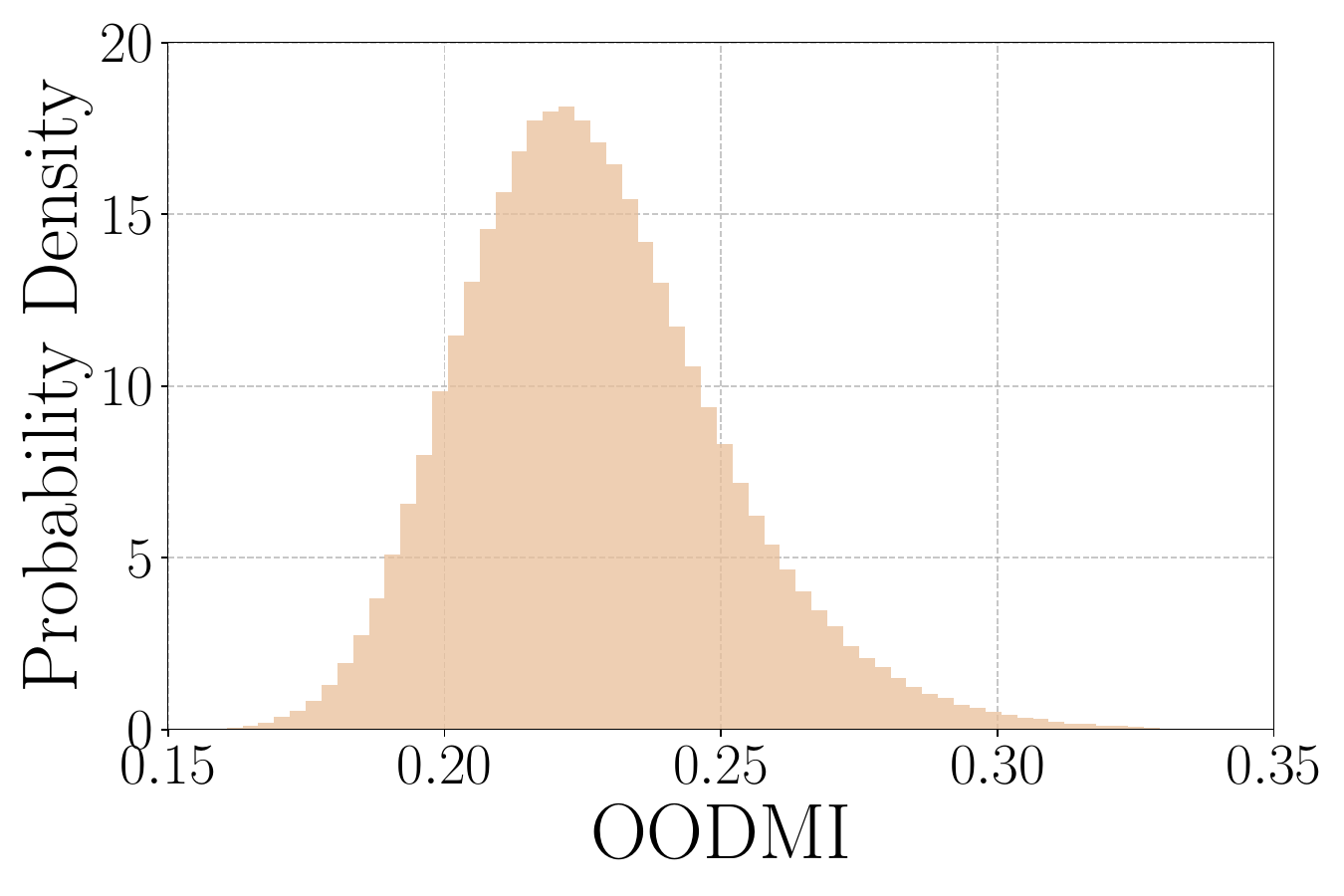}
     \end{subfigure}
     \hfill
     \begin{subfigure}[b]{0.32\textwidth}
         \centering
         \includegraphics[width=\textwidth]{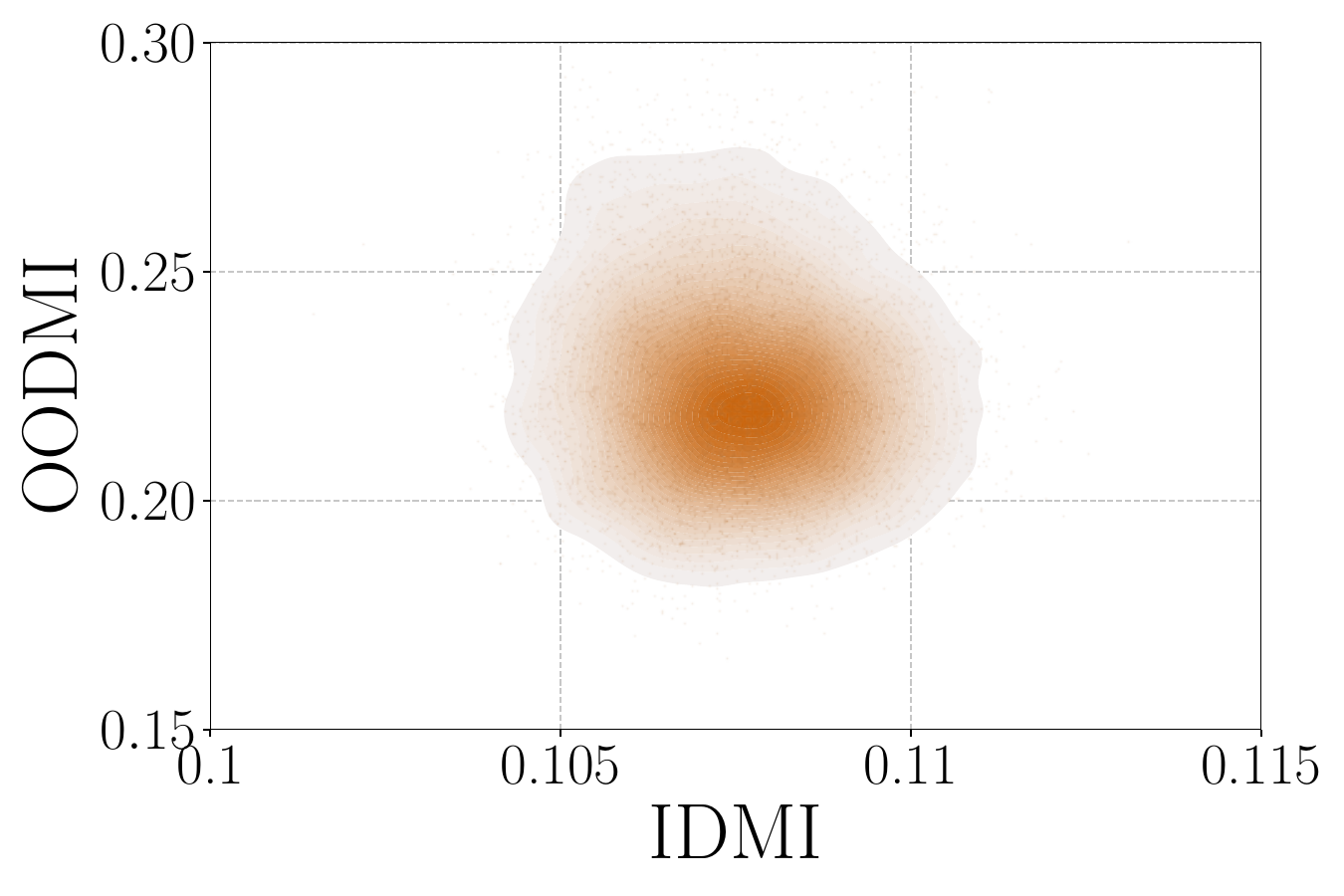}
     \end{subfigure}
        \caption{\textbf{Experiments show no hint of a functional collapse between couples of independently trained ResNet-18 on CIFAR-100.} Moreover, the in-distribution and out-of-distribution mutual information (IDMI, resp. OODMI) exhibit different variances but do not seem correlated.}
        \label{fig:functional_collapse_c100}
\end{figure}

Given the potentially very high number of equivalent modes due to permutation symmetries (see Section~\ref{sec:thm}), we support broadening the concept of collapse in the parameter-space (e.g., in \cite{d2021repulsive}) to \emph{functional collapse} to account for the impact of symmetries on the posterior%
. Parameter-space collapse is more restrictive and may not be formally involved when ensemble members lack diversity. It is also much harder to characterize as it would require an analysis of the loss landscape, at the very least.

We propose here to quantify functional collapse as a potential ground for the need for more complex repulsive ensembling methods~\citep{masegosa2020learning,rame2021dice}. To this extent, we randomly select 1000 ResNet-18 (trained to estimate the Bayesian posterior in Section~\ref{sec:comparison}) and compute the mean over the test set of their pairwise mutual information, quantifying the divergence between the single models and their average. We make these measures on in-distribution and out-of-distribution data (CIFAR-100 (ID) and SVHN (OOD)). In Figure~\ref{fig:functional_collapse_c100} (left), we see that the in-distribution MI between any two networks has a very low variance. There is, therefore, an extremely low probability of training %
two similar networks. This may be explained by the high complexity of the network (here, a ResNet-18), and we refer to Appendix~\ref{sec:functional_collapse_details} for results on a smaller architecture. In Figure~\ref{fig:functional_collapse_c100} (center), we see that the variance of the out-of-distribution MI is higher. However, we also note in Figure~\ref{fig:functional_collapse_c100} (right) that in contrast to common intuition (and to results for simpler models), we have, in this case, no significant correlation between the in-distribution and the out-of-distribution MI. This highlights that measuring the in-distribution \emph{diversity} (here with the MI) may be, in practice, a very poor indicator of the OOD detection performance of a model. Moreover, these results indicate that the complexity of the posterior is orders of magnitude higher than what we understand when taking symmetries into account.

\subsection{Frequency of weight permutations during training}
\label{disc:perm}
\begin{figure}[t]
    \centering
    \begin{subfigure}[b]{0.48\textwidth}
         \centering
         \includegraphics[width=\textwidth]{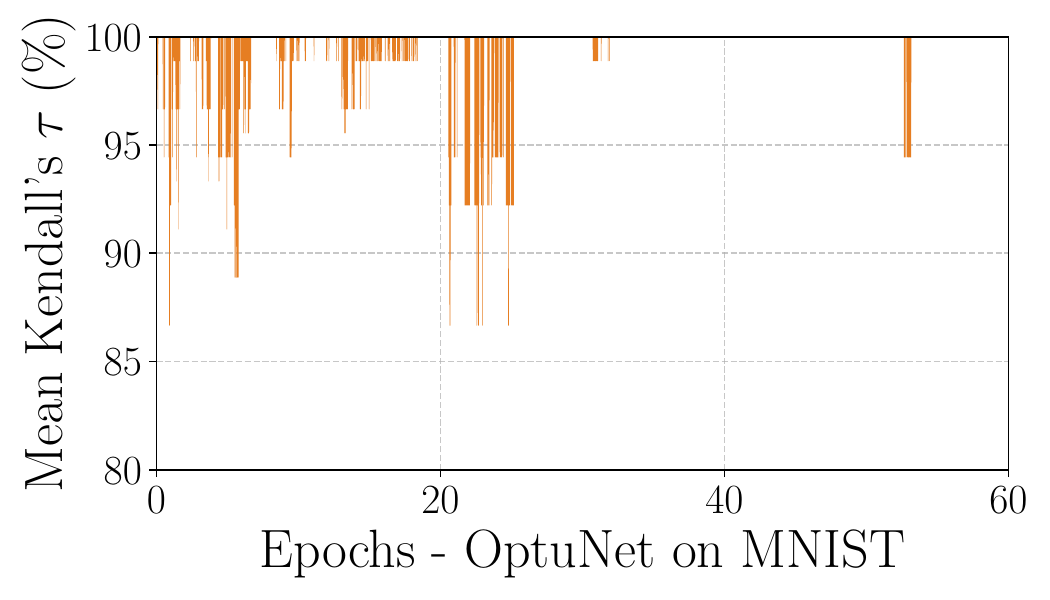}
     \end{subfigure}
     \hfill
     \begin{subfigure}[b]{0.48\textwidth}
         \centering
         \includegraphics[width=\textwidth]{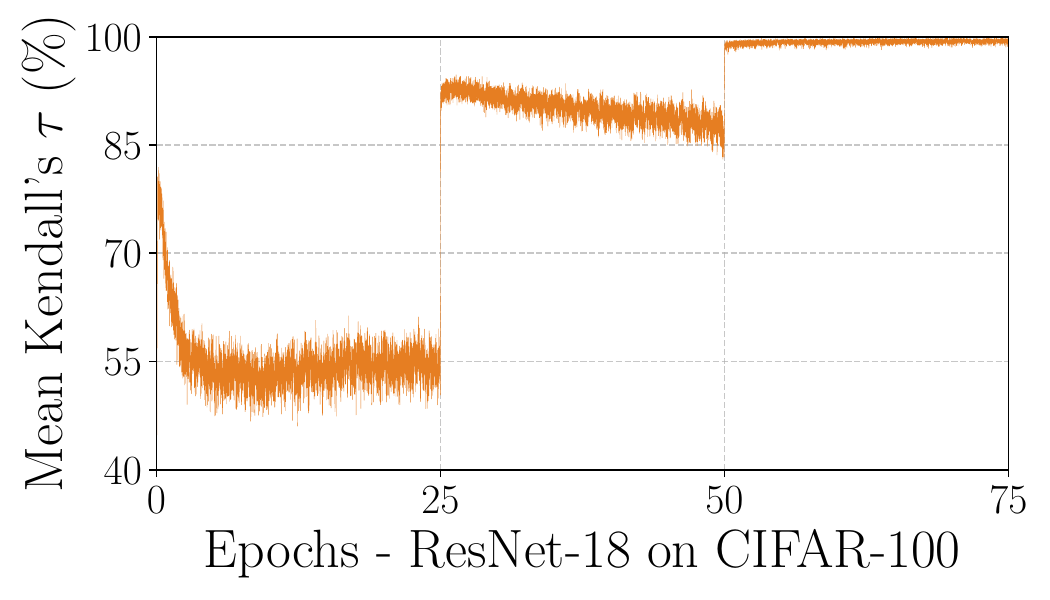}
     \end{subfigure}
     \caption{\textbf{Evolution during training of the mean Kendall's $\tau$ correlation between the permutations towards the identifiable model for all successive steps}: The correlation between successive permutations increases when the learning rate decreases.}
    \label{fig:kendal_tau}
\end{figure}     

We propose a new protocol to evaluate if a network tends to permute during training. Given a DNN $f_\vomega$, we can compute, for each step $s$ of the training, the permutation set $\Pi_s$ sorting its weights (and removing the symmetries). If the DNN tends to permute during the training, this implies a variation in the $\Pi_s$. We suggest to measure the extent of the variations using the Kendall's $\tau$ correlation coefficient~\citep{kendall1938new} between successive permutations $\Pi_s$ and $\Pi_{s+1}$. We plot the variation of the mean over several training instances and the elements of the permutation sets of Kendall's $\tau$ in Figure~\ref{fig:kendal_tau}. We see that on MNIST (left), the variations of the permutation set are scarce yet gathered around points of instability. These instabilities are due to the sorting mechanism based on the maximum values of the weights of the neurons. We have tried other statistics on the values of the weights, but taking the maximum seems the most stable. We see that the weights nearly never permute in the last phase of the training. The results differ for the ResNet-18 (right) since the number of degrees of freedom is much higher. We see a lot of variation during the phases with a high learning rate (reduced after 25 and 50 epochs). However, as for the first case, we do not see any particular sign of permutations in the last part of the training.

\section{Conclusion}

In this study, we have examined Bayesian neural network posteriors, which are pivotal for understanding uncertainty. Our findings suggest that part of the complexity in these posteriors can be attributed to the non-identifiability of modern neural networks viewing the posterior as a mixture of permuted distributions. To explore this further, we have introduced the \emph{min-mass} problem to investigate the real impact of scaling symmetries. Using real-world applications, we have proposed a method to assess the quality of the posterior distribution and its correlation with model performance, particularly in terms of uncertainty quantification.

While considering symmetries has provided valuable insights, our discussions hint at a deeper complexity going beyond weight-space symmetries. In future work, we plan to continue our exploration of this intriguing area.

\clearpage
\section{Reproducibility statement}

We use publicly available datasets, including MNIST, FashionMNIST, CIFAR100, SVHN, ImageNet-200, and Textures, to ensure transparency and accessibility. Please refer to Appendix~\ref{app:dataset_details} for details on these datasets. Our detailed experimental methods are outlined in Appendices \ref{sec:ex_details} and \ref{sec:exp_details}, and the proofs for all theoretical results are provided in Appendix~\ref{app:formalism}.

To help the replication of our work, we will share the source code for our experiments on GitHub shortly. Notably, we will release a library that tackles the non-identifiability of neural networks.

For our experiments in Section ~\ref{sec:comparison}, we rely exclusively on open-source libraries, such as the GitHub repository \href{https://github.com/JavierAntoran/Bayesian-Neural-Networks}{Bayesian-Neural-Networks} for SGHMC, BLiTZ~\citep{esposito2020blitzbdl} for variational Bayesian neural networks, and Laplace~\citep{daxberger2021laplace}. For the SWAG method, we also use the publicly available code from the original paper~\citep{maddox2019simple}. Finally, we estimate the maximum mean discrepancies with a torch version of the code from \cite{schrab2021mmd} and solved our convex optimization problems (see Definition \ref{def:scaled_problem}) with \texttt{cvxpy}~\citep{diamond2016cvxpy,agrawal2018rewriting}. The statistical experiments, such as Pearson's $\rho$ and Kendall's $\tau$, are performed with SciPy~\citep{2020SciPy-NMeth}. 

\section{Ethics}

Our primary goal in this paper is to improve our comprehension of the Bayesian posterior, which we argue is a fundamental element to understand to contribute to the reliability of machine-learning methods.

We note that training a substantial number of checkpoints for estimating the posterior, especially in the case of the thousand models trained on TinyImageNet, was energy intensive (around 3 Nvidia V100 hours per training). To mitigate the environmental impact, we opted for a carbon-efficient cluster.

\clearpage
\bibliography{iclr2024_conference}
\bibliographystyle{iclr2024_conference}
\clearpage
\appendix

\startcontents
\renewcommand\contentsname{Table of Contents - Supplementary Material}
{
\hypersetup{linkcolor=black}
\printcontents{ }{1}{\section*{\contentsname}}{}
}

\clearpage

\section{Details on the starting example}
\label{sec:ex_details}

\begin{figure}[t]
    \centering
    \includegraphics[width=0.5\linewidth]{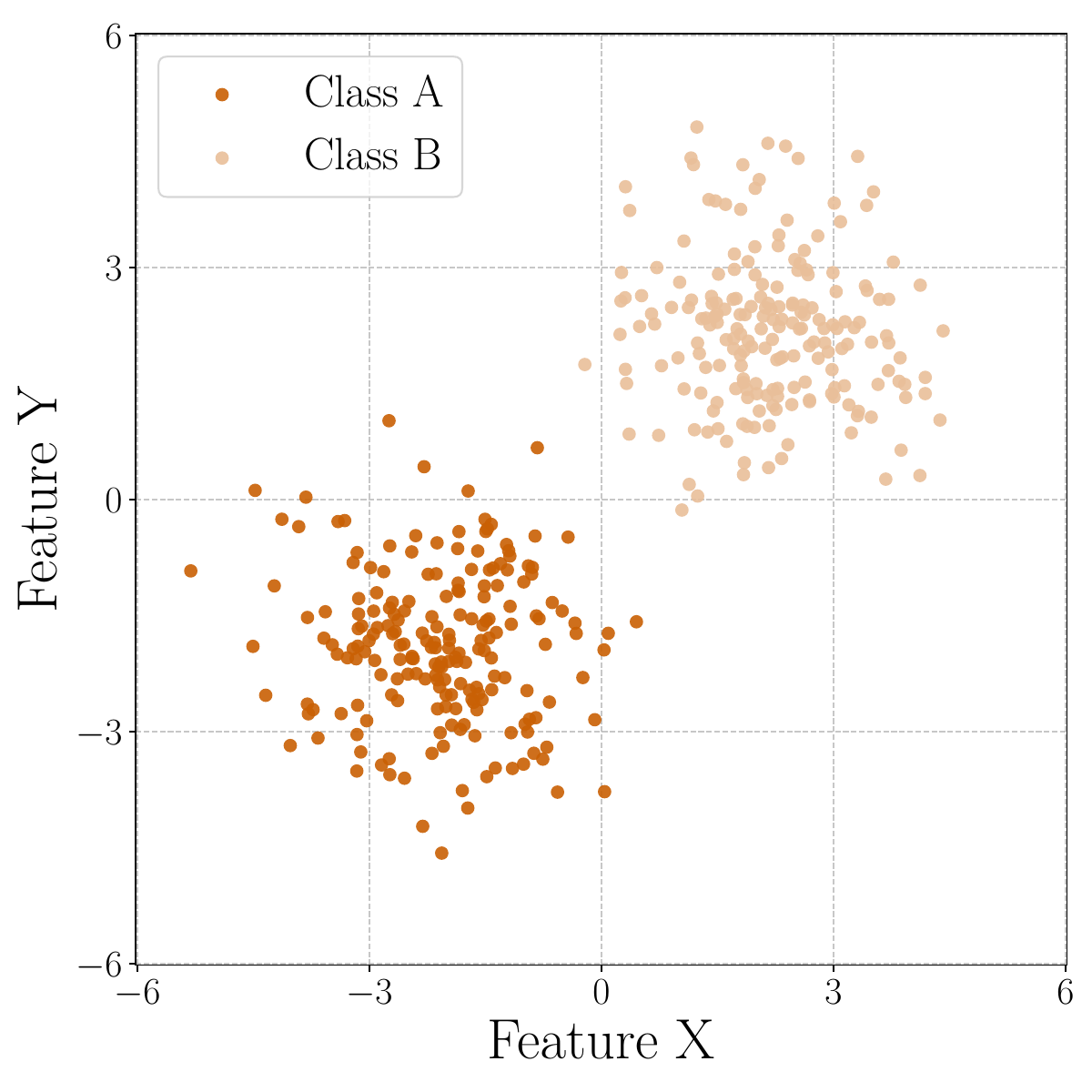}
    \caption{\textbf{The training data of the models whose posterior is represented on the introductory Figure \ref{fig:example}.} The data is separable to ensure efficient training of very simple 9-parameter perceptrons.}
    \label{fig:example_data}
\end{figure}

\subsection{Training of the two-layer perceptrons}
To create the introductory example, presented in Figure~\ref{fig:example}, we generated the data corresponding to the two classes from two different normal distributions of means $\vmu_A = (-2, -2)$ and $\vmu_B=(2,2)$ and of identity covariances. We ensured that the samples (200 points each) were fully separable to make training as simple as possible. We plot the training data in Figure~\ref{fig:example_data}. We then trained 10,000 two-layer perceptrons with two input, two hidden, and one output neurons with early stopping for 10 epochs with SGD and a binary cross-entropy loss. For the posterior to exhibit scaling symmetries, we separated the first and second layers with a ReLU activation function. We used a batch size of 10 data and a learning rate of 2. Finally, we selected networks that had a sufficiently low loss on the training set, therefore removing the few outliers (less than 1\%) that had not been trained successfully.

\subsection{Removing the symmetries}

Figure~\ref{fig:example} shows the density of weights of the last layer (excluding the last bias). The first Figure shows the unaltered projection of the posterior on these weights, which is then transformed to guarantee a 3-norm per neuron. We chose 3 as norm for visual purposes since the last layer (whose posterior is plotted in Figure~\ref{fig:example}) is not normalized and is subject to the normalization of the previous layer. Finally, we remove the permutation symmetries by ordering the weights. Contrary to the formalism developed in the following sections, we decide here to order the weights starting from the last layer (and not the first) to convey the message more efficiently. For each network, we check with random inputs that all symmetry removals do not alter the networks. Notebooks are available in the supplementary material of the submission and detail the process for generating the Figure. We provide more detail on the symmetry removal algorithms in Section~\ref{sec:alg_sym_removal}.
\clearpage

\section{Details on the experiments of section \ref{sec:comparison}}
\label{sec:exp_details}
\subsection{Experimental details}

In this section, we develop the training recipes of the different models as well as the parameters used for all the posterior estimation methods. We train all of our models with PyTorch~\citep{paszke2019pytorch} on V100 clusters. All networks are trained until the last step to avoid biasing the posterior with the validation set.

For all the variational Bayesian Neural Networks (vBNN), we use the default priors from Bayesian layers in torch zoo (Blitz)~\citep{esposito2020blitzbdl}.

When using ResNet-18 models, we perform last-layer approximations of Laplace and Dropout. For more information on last-layer approximation, you may refer to \cite{brosse2020last} for instance.

\paragraph{OptuNet - MNIST}

We train OptuNet for 60 epochs with batches of size 64 using stochastic gradient descent (SGD) with a start learning rate of 0.04 and a weight decay of $2\times10^{-4}$. We decay the learning rate twice during training, at epochs 15 and 30, dividing the learning rate by 2. 

We train the vBNNs with the same number of epochs, albeit using 3 Monte Carlo estimates of the Evidence lower bound (ELBO) at each step and using a Kullback-Leibler divergence (KLD) of $10^{-5}$. We disable weight decay for the training of the vBNNs.

To compute the Maximum Mean Discrepancies, we use a full Hessian Laplace optimization. We use a dropout rate of 0.2 on the last layer (both for training and testing) and perform SWAG using 20 different models (also set as the maximum number to keep a low-rank matrix), continuing the training with the classic high-learning rate schedule (starting with a Linear increase) for twice as long. We collect the models 20 times with 10-epoch intervals between each of them. We sample the models with a scale of 0.1.

For the dataset, we normalize the data as usual and perform a random crop of size 28 and padding four during training. We use the normalized test images for testing. 

For more details on the architecture of our OptuNet, please refer to Section~\ref{sup:arch_detail}.

\paragraph{ResNet-18 - CIFAR-100}

We train the ResNet-18 for 75 epochs with batches of size 128 using SGD with Nesterov~\citep{nesterov1983method,sutskever2013importance}, with a start learning rate of 0.1, a momentum of 0.9, and a weight decay of $5\times10^{-4}$. Similarly to MNIST, we decay the learning rate twice during training, this time at epochs 25 and 50, and divide the learning rate by 10.

We train the variational BNN with SGD for 150 epochs, with a starting learning rate of 0.01, which we decay once after 80 epochs. We weigh the KLD with a coefficient $2\times10^{-9}$ and perform three ELBO samples at each step.

To compute the MMDs for Laplace methods, we use the last-layer Kronecker approximation. Indeed, the last layer of ResNet-18 is too large for its full Hessian to fit in memory. Moreover, Kronecker is not available for the full network as it is not implemented for batch normalization layers, and the low-rank version is too long to train. For the dropout models, we perform last-layer dropout of probability 0.5. As for MNIST, we train the SWAG models twice the time of the original training recipe with the usual settings (start learning rate with a linear ramp) and collect 20 models with 10-epoch intervals. We sample the models with a scale of 0.1.

For CIFAR-100, we perform the classic normalization, as well as a random crop for a four-pixel padding and a random horizontal split.

\paragraph{ResNet-18 - TinyImageNet}
We train the ResNet-18 for 200 epochs with a batch size of 128 using SGD with a start learning rate of 0.2, a weight decay of $10^{-4}$, and a momentum of 0.9. We use a cosine annealing scheduler until the end of the training.

We did not manage to train vBNNs for TinyImageNet. %
The other methods use exactly the same hyperparameters as for ResNet-18 on CIFAR-100.

For TinyImageNet, we perform the same preprocessing as for CIFAR-100, except we keep the original resolution of $64\times64$ pixels.

\subsection{Detailed architecture of OptuNet}
\label{sup:arch_detail}

\begin{figure}[t]
    \centering
    \includegraphics[width=0.2\linewidth]{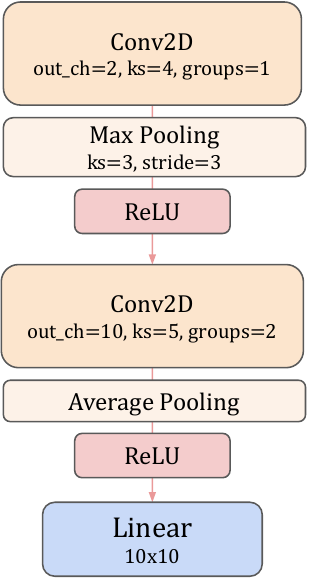}
    \caption{\textbf{Detailed architecture of OptuNet}. OptuNet includes only 392 parameters yet achieves between 85\% and 90\% accuracy on MNIST. Only the last fully connected layer contains biases.}
    \label{fig:optunarch}
\end{figure}

For OptuNet, we aimed to create the smallest number of parameters possible while keeping decent performance on MNIST. To this extent, we performed an architecture search with tree-structured Parzen estimators (e.g., \citep{bergstra2011algorithms}) in Optuna~\citep{akiba2019optuna}. To reduce the number of parameters, OptuNet uses grouped convolutions, introduced by \cite{krizhevsky2012imagenet} (see \cite{laurent2023packed} for a formal definition) as its second convolutional layer.

\subsection{Details on the Datasets}
\label{app:dataset_details}

This part provides some details on the different datasets used throughout the paper.

\subsubsection{In distribution datasets}
In this work, we used the following datasets for training and testing.

\paragraph{MNIST}
The MNIST dataset~\cite{lecun1998gradient} comprises 70,000 binary images of handwritten digits, each of size 28$\times$28 pixels. These images are divided into a training set with 60,000 samples and a testing set with 10,000 samples.

\paragraph{CIFAR-100}
CIFAR-100~\cite{krizhevsky2009learning} consists of 100 object classes and contains a training set with 50,000 images and a testing set including 10,000 images. Each of the images is RGB and of size 32$\times$32 pixels.

\paragraph{Tiny-ImageNet}
TinyImageNet is a subset of ImageNet-1k~\cite{deng2009imagenet}. It consists of 200 object classes, each with 500 training and 50 validation images. Additionally, there are 50 test images per class for evaluating models. The images are RGB and have a size of 64$\times$64 pixels.

\subsubsection{Out-of-distribution datasets}

The objective of the following out-of-distribution datasets is to evaluate the quality of the posterior as a means to quantify the epistemic uncertainty. In this work, we do not consider these datasets as fully representative of difficult real-world out-of-distribution detection tasks.

\paragraph{Fashion-MNIST}
The Fashion-MNIST dataset~\citep{xiao2017fashion} is a set of 28$\times$28 pixel-grayscale images consisting of 60,000 training and 10,000 test images. We used the test set as is for the out-of-distribution tasks with MNIST.

\paragraph{SVHN}
The Street View House Numbers (SVHN) dataset~\citep{netzer2011reading} is a large-scale dataset of 600,000 images of digits obtained from house numbers in Google Street View. In our work, we kept a fixed set of 10,000 images that we cropped to squares of 32$\times$32 pixels.

\paragraph{Textures}
The Describable Textures Dataset~\citep{cimpoi14describing} is a dataset containing 5640 images of textures divided into three subsets. Considering that the number of images is limited compared to the other testing sets, we use the concatenation of all the subsets for OOD detection. We resize all the images to 64$\times$64 pixels to stick to the size of TinyImageNet.

\subsection{Full MMD Table}

\begin{table}[t]
\centering
\resizebox{0.65\textwidth}{!}{%
\begin{tabular}{ccccccccc}
\hline
 &  & Method & \multicolumn{3}{c}{MMD} & \multicolumn{3}{c}{NS MMD} \\
\multicolumn{1}{l}{} & \multicolumn{1}{l}{} & \multicolumn{1}{l}{} & \multicolumn{1}{l}{Median} & \multicolumn{1}{l}{Mean} & Max. & \multicolumn{1}{l}{Median} & \multicolumn{1}{l}{Mean} & Max. \\ \hline
\multirow{10}{*}{\rotatebox[origin=c]{90}{MNIST - OptuNet}} & \multirow{5}{*}{\rotatebox[origin=c]{90}{One Mode}} &
      Dropout & 14.9 & 14.6 & 22.9 & 14.3 & 14.3 & 22.5 \\
 &  & BNN & 18.8 & 18.3 & 29.5 & 17.1 & 17.0 & 27.4 \\
  &  & SGHMC & 16.7 & 16.6 & 27.8 & 17.7 &18.3  & 32.2 \\
 &  & SWAG & 15.9 &  15.8& 25.9 & 14.6 & 13.8 & 21.6 \\
 &  & Laplace & 10.6 & 10.5 & 19.1 & 9.5 & 9.4 & 16.4 \\ \cline{2-9} 
 & \multirow{5}{*}{\rotatebox[origin=c]{90}{Multi Mode}} & 
      Dropout & 2.1 & 2.3 & 4.1 & 2.1 & 2.0 & 3.5 \\
 &  & BNN & 2.8 & 2.9 & 5.8 & 2.5 & 2.4 & 4.3 \\
 &  & SWAG & 1.8 & 2.0 & 4.1 & 1.3 & 1.3 & 2.5 \\
 &  & Laplace & 1.8 & 2.0 & 4.1 & 0.8 & 0.8 & 1.5 \\
 &  & DE & 0.0 & 0.0 & 0.0 & 0.0 & 0.0 & 0.0 \\ \hline
 \multirow{10}{*}{\rotatebox[origin=c]{90}{CIFAR-100 - ResNet-18}} & \multirow{5}{*}{\rotatebox[origin=c]{90}{One Mode}} &
      Dropout & 4.5 & 4.6 & 7.2 & 7.4 & 7.5 & 12.0 \\
 &  & BNN & 9.0 & 9.2 & 14.6 & 10.0 & 10.2 & 16.0 \\
  &  & SGHMC & 7.5 & 7.3 & 11.0 & 7.9 & 8.0 & 13.0 \\
 &  & SWAG & 6.7 & 7.0 & 11.2 & 7.2 & 7.5 & 12.3 \\
 &  & Laplace & 5.7 & 6.0 & 9.7 & 7.0 & 7.2 & 11.7 \\ \cline{2-9} 
 & \multirow{5}{*}{\rotatebox[origin=c]{90}{Multi Mode}} 
    & Dropout & 0.7 & 0.7 & 1.2 & 4.4 & 4.5 & 7.5 \\
 &  & BNN & 6.1 & 6.3 & 10.6 & 5.6 & 5.8 & 10.0 \\
 &  & SWAG & 5.0 & 5.2 & 8.4 & 5.4 & 5.6 & 9.2 \\
 &  & Laplace & 0.6 & 0.6 & 1.0 & 4.3 & 4.4 & 7.2 \\
 &  & DE & 0.0 & 0.0 & 0.0 & 0.0 & 0.0 & 0.0 \\ \hline
 \multirow{10}{*}{\rotatebox[origin=c]{90}{TinyImageNet - ResNet-18}} & \multirow{5}{*}{\rotatebox[origin=c]{90}{One Mode}} &
      Dropout & 9.5 & 9.6 & 15.0 & 4.9 & 5.0 &7.8  \\
 &  & BNN & / & / & / & / & / & / \\
 &  & SGHMC & 9.8 & 10.0 & 15.7 & 5.3 & 5.2 & 7.9 \\
 &  & SWAG & 9.1 & 9.5 & 15.6 & 5.4 & 5.6 & 9.5 \\
 &  & Laplace & 5.5 & 5.8 & 9.6 & 6.1 & 6.3 & 10.1 \\ \cline{2-9} 
 & \multirow{5}{*}{\rotatebox[origin=c]{90}{Multi Mode}} 
    & Dropout & 4.3 & 4.4 & 7.4 & 1.8 & 1.8 & 3.0 \\
 &  & BNN & / & /& / & / & / & / \\
 &  & SWAG & 6.7 & 6.9 & 11.3 & 3.9 & 4.2 & 7.3 \\
 &  & Laplace & 0.6 & 0.6 & 9.9 & 3.1 & 3.2 & 5.3 \\
 &  & DE & 0.0 & 0.0 & 0.0 & 0.0 & 0.0 & 0.0 \\ \hline
\end{tabular}%
}
\caption{\textbf{Comparison of popular methods approximating the Bayesian posterior.} Mean Maximum Discrepancies for OptuNet are expressed in \%. For ResNet-18 networks, MMDs are expressed in \perthousand. NS MMD is the MMD computed after the removal of the symmetries. Multi-mode methods use ten independently trained models. DE stands for Deep Ensembles~\citep{simple2017lakshminarayanan}.}
\label{tab:mmd-table}
\end{table}

In Table~\ref{tab:mmd-table}, we report the weighted average over the layers of the medians, means, and maxima of the Maximum Mean Discrepancies of each layer, computed in Section~\ref{sec:comparison}. We recall that the maximum is used to obtain the best discriminative power of two-sample tests, for instance, in ~\citep{gretton2012kernel,schrab2021mmd}. We report the median in the main table to improve the representativity of the values. However, we note that the order of the values does not seem to change, particularly between the means, medians, and maxima.

We see that, after Deep Ensembles - that have perfect MMD as expected - Laplace methods perform best on all measures, be it with or without symmetries, as well as for all types of architectures. However, as shown in Table~\ref{tab:main-table}, they are inferior to SWAG on performance metrics. This hints that the correlation between the estimated quality of the posterior and the real-world metrics may not be fully correlated. Please note that for Deep Ensembles, we used different checkpoints for the second sample than in the sample corresponding to the estimation of the posterior.

\clearpage
\section{Detailed formalism, propositions, and proofs}
\label{app:formalism}
This section details the definitions, properties, and propositions and provides sketches of proofs. %

Throughout the paper, we use an intuitive definition of neural networks that we propose to formalize in the following definition:

\begin{definition}[Neural network]
Let $\vx\in\mathbb{R}^{d}$ be an input datum, $r$ the rectified linear unit, and $s$ an almost everywhere differentiable activation function. We define the $L$-layer neural network $f_\theta$ with $\theta=\{W^{[l]}, b^{[l]}\}_{l\in\llbracket1,L\rrbracket}$ as follows:
\begin{align}
&\va^{[0]} = \vx \\
&\forall l\in\llbracket1, L\rrbracket, \ \vz^{[l]} = W^{[l]}\va^{[l-1]}+b^{[l]} \label{eqdef:linear} \\
&\forall l\in\llbracket1, L-1\rrbracket, \ \va^{[l]} = r(\vz^{[l]}) \\
&f_\theta(\vx) = s(\vz^{[l]})
\end{align}
\end{definition}

This definition of neural networks is limited to standard multilayer perceptions. However, it can be easily extended to a mix of convolutional and fully connected layers to the cost of a heavier formalism. To do this, we should replace the relationship between $\vz^{[l]}$ and $\va^{[l-1]}$ in \Eqref{eqdef:linear} by a sum on the kernels of the cross-correlation operator. It would also be possible to account for residual connections using the computation graph.

\begin{property}[Gradients and backpropagation]
\label{prop:backprop}
    Let us denote the training loss $\mathcal{L}(\theta)$ of the $L$-layer neural network $f_\theta$ and $\theta=\{W^{[l]}, b^{[l]}\}_{l\in\llbracket1,L\rrbracket}$. With, for all $l\in\llbracket1,L\rrbracket$, $\vdelta^{[l]}$ as the gradient of $\mathcal{L}(\theta)$ with respect to $\vz^{[l]}$, we have,
    \begin{align}
        &\vdelta^{[L]} = \nabla_{\vz^{[L]}}\mathcal{L}(\theta)\\
        &\forall l\in\llbracket1,L-1\rrbracket, \vdelta^{[l]}= \nabla_{\vz^{[l]}}\mathcal{L}(\theta) = \left(W^{[l+1]^\intercal}\vdelta^{[l+1]}\right) \circ \nabla_{\vz^{[l]}} \va^{[l]}\\
        &\forall l\in\llbracket1,L\rrbracket, \nabla_{W^{[l]}}\mathcal{L}(\theta) = \vdelta^{[l]}\va^{[l-1]^\intercal} \\
        &\forall l\in\llbracket1,L\rrbracket, \nabla_{b^{[l]}}\mathcal{L}(\theta) = \vdelta^{[l]}.
    \end{align}
\end{property}

\subsection{Extended formalism on symmetries}

In this section, we propose several definitions and properties to extend the formalism quickly described in the main paper.

\subsubsection{Scaling symmetries}

We start with a recall of the definitions of the line-wise and column-wise products between a vector and a matrix.
\begin{definition}
We denote the line-wise product as $\bigtriangledown$ and the column-wise product as $\rhd$. Let $\vomega \in \mathbb{R}^{n\times m}$, $\vlambda \in \mathbb{R}^m$, and $\vmu\in \mathbb{R}^n$, with $n, m \geq 1$, 
\begin{equation}
    \forall i\in\llbracket 1, n\rrbracket, j\in\llbracket 1,m\rrbracket, \ (\vlambda \bigtriangledown \vomega)_{i,j} = \lambda_j\omega_{i,j} \ \text{and} \ (\vmu \rhd \vomega)_{i,j} = \mu_i\omega_{i,j}.
\end{equation}
\end{definition}

To handle biases, we extend these definitions to the product of vectors, considering the right vector as a one-dimensional matrix. It follows that for all vectors $\vlambda, \va \in \mathbb{R}^m$,
\begin{equation}
    \vlambda \bigtriangledown \va = \vlambda \rhd \va = \vlambda \cdot \va.
\end{equation}

The line-wise and column-wise products between vectors and matrices \emph{have priority} over the classic matrix product "$\times$" but not over the element-wise product between vectors "$\cdot$". We will include parentheses when deemed necessary for the clarity of the expressions.

\paragraph{Non-negative homogeneity of the rectified linear unit} 
The rectified linear unit~\citep{nair2010rectified}, used as the standard activation function in most networks (with the notable exception of vision transformers~\citep{dosovitskiy2020image} %
is non-negative homogenous and therefore authorizes scaling symmetries~\citep{neyshabur2015path}.

\begin{property}
Denote $r$ the rectified linear unit such that $r(x)=x_{+}$. For all vectors $\vlambda\in (\mathbb{R}^{+})^d$ and $\vx\in\mathbb{R}^d$, $r$ is positive-homogeneous, i.e.
\begin{equation}
    r(\vlambda \cdot \vx) = \vlambda \cdot r(\vx).
\end{equation}
\end{property}

\paragraph{Some properties on line-wise and column-wise products} In this paragraph, we propose some simple properties to help achieve an understanding of these notations.
\begin{property}
For all vectors $\vlambda^{(1)}, \vlambda^{(2)} \in \mathbb{R}^n$, $\vtheta \in \mathbb{R}^{\cdot\times n}$, and $\vomega\in \mathbb{R}^{n\times \cdot}$
   \begin{gather}
   (\vlambda^{(1)}\bigtriangledown \vtheta) \times (\vlambda^{(2)}\rhd\vomega) = (\vlambda^{(1)}\cdot\vlambda^{(2)})\bigtriangledown \vtheta \times \vomega=\vtheta \times (\vlambda^{(1)}\cdot\vlambda^{(2)})\rhd \vomega
\end{gather} 
\end{property}

\begin{proof}
Let us denote $A = (\vlambda^{(1)}\bigtriangledown \vtheta) \times (\vlambda^{(2)}\rhd\vomega)$
\begin{align}
A_{i,j}= \sum\limits_{k=1}^{n} \lambda_k^{(1)}\theta_{i, k} \lambda_k^{(2)}\omega_{k, j}=& \sum\limits_{k=1}^{n} \theta_{i, k} \lambda^{(2)}_k\lambda^{(1)}_k\omega_{k, j}=(\vtheta\times(\vlambda^{(1)}\cdot\vlambda^{(2)})\rhd\vomega)_{i,j}\\
=& \sum\limits_{k=1}^{n} \lambda^{(2)}_k\lambda^{(1)}_k \theta_{i, k} \omega_{k, j}=((\vlambda^{(1)}\cdot\vlambda^{(2)})\bigtriangledown \vtheta \times \vomega)_{i,j}
\end{align}
\end{proof}

However, please note that we do not have $\vlambda^{(1)}\bigtriangledown (\vlambda^{(2)}\rhd\vomega) = (\vlambda^{(1)}\cdot\vlambda^{(2)})\bigtriangledown \vomega$ in general.

\begin{property}
For all vectors $\vlambda^{(1)} \in \left(\mathbb{R}_{\neq 0}\right)^n$, $\vlambda^{(2)} \in \mathbb{R}^n$, $\vtheta \in \mathbb{R}^{\cdot\times n}$, and $\vomega\in \mathbb{R}^{n\times \cdot}$
\begin{gather}
   (\vlambda^{(1)}\bigtriangledown \vtheta) \times r(\vlambda^{(2)}\rhd\vomega)=\vtheta \times r\left((\vlambda^{(1)}\cdot\vlambda^{(2)})\rhd \vomega\right)
\end{gather}
\end{property}

\begin{proof}
Let us denote $B=(\vlambda_1\bigtriangledown \vtheta) \times r(\vlambda_2\rhd\vomega)$
\begin{align}
B_{i,j}=& \sum\limits_{k=1}^{n} \lambda^{(1)}_k \theta_{i, k} r(\lambda^{(2)}_k\omega_{k, j}) \\
=& \sum\limits_{k=1}^{n}\theta_{i, k} r( \lambda^{(1)}_k \lambda^{(2)}_k \omega_{k, j}) \\
=& (\vtheta \times r((\vlambda^{(1)}\cdot\vlambda^{(2)})\rhd\vomega))_{i,j}
\end{align}
\end{proof}

This property is crucial for scaling symmetries, which are one of its special cases. Indeed, for all vectors $\vlambda \in \left(\mathbb{R}_{>0}\right)^n$,
\begin{equation}
    \label{eq:simple_scaling_sym}
    (\vlambda^{-1}\bigtriangledown \vtheta) \times r(\vlambda\rhd\vomega) = \vtheta \times r(\vomega).
\end{equation}

Finally, we can gather the previous results to extend \Eqref{eq:simple_scaling_sym} to weights and biases in the following property.

\begin{property}
For all vectors $\vlambda \in \left(\mathbb{R}_{>0}\right)^n$, $\vb \in \mathbb{R}^n$, $\vtheta \in \mathbb{R}^{\cdot\times n}$, and $\vomega\in \mathbb{R}^{n\times \cdot}$
    \begin{gather}
   (\vlambda^{-1}\bigtriangledown \vtheta) \times r(\vlambda\rhd(\vb + \vomega)) =  \vtheta \times r(\vb + \vomega)
\end{gather}
\end{property}

\begin{proof}
\begin{align}
(\vlambda^{-1}\bigtriangledown \vtheta) \times r(\vlambda\rhd(\vb+\vomega)) & = (\vlambda^{-1}\bigtriangledown \vtheta) \times r(\vlambda \cdot \vb+\vlambda\rhd\vomega) \\
& =   \vtheta \times r(\vlambda^{-1}\vlambda \cdot \vb + \vlambda^{-1}\vlambda\rhd\vomega) \\
&=  \vtheta \times r(\vb + \vomega)
\end{align}
\end{proof}

This property is provided for the simplest case with only two matrices. However, it can be trivially extended to an arbitrary number of matrices, chaining the different line-wise and column-wise multiplications. It follows that we can extend the definition of scaling symmetries to deeper networks.

Moreover, we can also extend the scaling symmetries to mixes of convolutional and linear networks. To do this, we first extend the definition of the line-wise and column-wise products to convolutional layers through the following definition.
\begin{definition}
\label{def:extension_scaling}
    Let the 4-dimensional tensor $\vomega\in\mathbb{R}^{C_{\mathrm{out}} \times C_{\mathrm{in}}\times k_v \times k_h}$ be the weight of a convolutional layer of input and output channels $C_{\mathrm{in}}$ and $C_{\mathrm{out}}$, and kernel $k_v \times k_h$. Denote $\vlambda\in\mathbb{R}^{C_{\mathrm{in}}}$ and $\vmu\in\mathbb{R}^{C_{\mathrm{out}}}$ We extend the operators $\rhd$ and $\bigtriangledown$ to products between vectors and 4-dimensional tensors and have,
    \begin{align}
        \forall c_{out}\in\llbracket 1, C_{\mathrm{out}}\rrbracket, c_{in}\in\llbracket 1,& C_{\mathrm{in}}\rrbracket, i\in\llbracket 1, n\rrbracket, j\in\llbracket 1,m\rrbracket, \\
        & (\vlambda \bigtriangledown \vomega)_{c_{out}, c_{in}, i,j} = \lambda_{c_{in}}\omega_{c_{out},c_{in},i,j}, \\
        & \text{and} \ (\vmu \rhd \vomega)_{c_{out}, c_{in},i,j} = \mu_{c_{out}}\omega_{c_{out},c_{in},i,j}.
    \end{align}
\end{definition}
With this definition, we keep the previous properties and enable chaining transformations on both convolutional and fully connected layers.

\paragraph{Equivalence to standard linear algebra}
\label{par:equiv_linear_algebra}
We propose the line-wise ($\bigtriangledown$) and column-wise ($\rhd$) notations for their intuitiveness. However, when applied to fully connected layers, they are equivalent to more common notations from linear algebra. Indeed, sticking to $\vomega \in \mathbb{R}^{n\times m}$, $\vlambda \in \mathbb{R}^m$, and $\vmu\in \mathbb{R}^n$, with $n, m \geq 1$, we have that 
\begin{equation}
    \vlambda \bigtriangledown \vomega = \mathrm{diag}(\vlambda)\times\vomega \ \text{and} \ \vmu \rhd \vomega = \omega\times\mathrm{diag}(\vmu).
\end{equation}

Furthermore, we can also write, with $\circ$ the Hadamard (elementwise) product between matrices,
\begin{equation}
    \vlambda\bigtriangledown\left(\vmu\rhd\vomega\right) = (\vlambda^\intercal\times\vmu)\circ\omega.
\end{equation}

These equivalences are important for the implementation of the log-log convex problem to minimize the L2 regularization term on the space of the scaling symmetries (see Section \ref{sec:minimum_mass_pbm}).

\subsubsection{Permutation symmetries}
\begin{definition}
 We define $P_n$ as the set of permutations of vectors of size $n$. $P_n$ contains $\left|P_n\right|=n!$ elements. The elements of $P_n$ are doubly-stochastic binary matrices from $\{0,1\}^n$. Finally, we have that, for all $\vpi_1\in P_n$, $\vpi_1 \times \vpi_1^\intercal=\mathbb{1}_n$.
\end{definition}

From this definition and property, we can directly derive the following result:
\begin{property}
For all $\vtheta\in\mathbb{R}^{\cdot\times m}$, $\vomega \in\mathbb{R}^{m\times n}$, and permutation matrix $\vpi \in P_m$,
\begin{equation}
     \forall\vx\in\mathbb{R}^n, \ \vtheta \vpi^\intercal \times r\left(\vpi \times \vomega\vx\right) = \vtheta \times r(\vomega\vx).
\end{equation}
\end{property}
\begin{proof}
    Let us take $m,n\in\mathbb{N}$, $\vomega\in\mathbb{R}^{m\times n}, \vb\in\mathbb{R}^{m}, \vx\in\mathbb{R}$, and $\vpi\in P_m$ a permutation matrix corresponding to the bijective mapping $\sigma$ in $\llbracket 1, m\rrbracket$. We have that,
\begin{equation}
    \left[\vpi(\vb + \vomega)\times\vx\right]_i = \vb_{\sigma(i)} + \sum\limits_{k=1}^{n} \omega_{\sigma(i),k} x_k.
\end{equation}
Hence, 
\begin{align}
    \left[\vpi^\intercal \times r(\vpi(\vomega)\times\vx)\right]_i &= \vb_{\sigma^{-1}(\sigma(i))} + \sum\limits_{k=1}^{n}r\left(\omega_{\sigma^{-1}(\sigma(i)),k}\cdot x_k\right) \\
    &= \left[r(\vb+\vomega\times\vx)\right]_i,
\end{align}
and it comes directly that,
\begin{equation}
    \vtheta \vpi^\intercal \times r\left(\vpi \times \vomega\vx\right) = \vtheta \times r(\vomega\vx).
\end{equation}
\end{proof}

This property is provided for the simplest case with only two matrices. However, it can be trivially extended to an arbitrary number of matrices, chaining the different permutations. As for the scaling symmetries, it follows that we can extend the definition of permutation symmetries to deeper networks.

Moreover, as for scaling symmetries in definition~\ref{def:extension_scaling}, we can extend permutation operations and symmetries to convolutional layers and their 4-dimensional tensors. To do this, we consider that the permutation matrices act on the output and input channels and permute the whole kernels.

\subsubsection{Softmax additive symmetry}
For the sake of completeness, we also recall the additive softmax symmetries in this section. We left this symmetry apart in our analysis as it involves, at most, one degree of freedom.

\begin{definition}
    We recall that the softmax function $\sigma$ is defined by the following equation, given the logits $\va\in\mathbb{R}^n$:
\begin{equation}
    \sigma(\va)_i = \frac{e^{a_i}}{\sum\limits_{j=1}^n e^{a_j}}
\end{equation}

\end{definition}

\begin{property}
    For all $\va,\vb\in\mathbb{R}^n$, $\lambda\in\mathbb{R}$, and + the point-wise sum when applied between vectors and scalars,
\begin{equation}
    \sigma(\va + \vb + \lambda) = \sigma(\va + \vb)
\end{equation}
\end{property}
\begin{proof}
Let us denote $S = \sigma(\va + \vb + \lambda)$,
    \begin{align}
        S_i =& \frac{e^{a_i +  b_i + \lambda}}{\sum\limits_{j=1}^n e^{a_j+ b_j + \lambda}} \\
        S_i =& \frac{e^\lambda e^{a_i +  b_i}}{ e^\lambda \sum\limits_{j=1}^n e^{a_j+ b_j}} \\
         S_i=& (\sigma(\va + \vb))_i
    \end{align}
\end{proof}

\subsection{Removing symmetries a posteriori}
\label{sec:alg_sym_removal}
\paragraph{Removing permutation symmetries}
Similar to \cite{pourzanjani2017improving}, we propose to remove permutation symmetries by ordering neurons according to the value of their first corresponding parameter. In dense layers, this first parameter is the weight of the first input neuron, and in convolutional layers, it corresponds to the top-left weight of the kernel of the first channel. This solution is more general than that of \cite{pourzanjani2017improving} since neural network layers do not always have biases. Specifically, practitioners often remove biases from a layer when it is followed by a batch normalization~\citep{ioffe2015batch} to minimize the number of computations of the inference and backpropagation steps.

\paragraph{Removing scaling symmetries}
In this paper, we propose two different types of algorithms to remove the scaling symmetries of a neural network \emph{a posteriori}. The first algorithm was proposed by \cite{neyshabur2015path} and normalizes the norm of the weights of each neuron in all the layers except the last. This is the algorithm that we use in Table \ref{tab:main-table} We also provide an implementation of this algorithm to scale the standard deviation of the weights to one. Furthermore, we define in Section \ref{sec:minimum_mass_pbm} a new problem that enables the deletion of scaling symmetries. We show the existence and uniqueness of the \emph{min-mass} problem. As such, it can also be used to remove the symmetries of any network. While scalable to current architectures, the convex optimization is slower than the normalization of the weights, and its behavior may be slightly more difficult to grasp. This is why we stick to the former algorithm when many computations are needed.

\paragraph{Removing softmax additive symmetries}
To remove this type of symmetry, it is sufficient to scale the sum of the biases of the last layer to some constant, for instance, to 1.

\subsection{Permutation-equivariance of the gradient of the loss}
Before coming to the equivariance of the training operator, we start by proving the following lemma.

\begin{lemma}
\label{lem:grad_equi}
    Let $f_{\theta}$ be a neural network trained with a loss $\mathcal{L}(\theta)$, the gradient of the loss $\nabla\mathcal{L}$ is permutation equivariant, i.e.,
    \begin{equation}
        \forall\Pi\in\mathbb{\Pi}, \ \nabla\mathcal{L}(\mathcal{T}_p(\theta, \Pi)) = \mathcal{T}_p(\nabla\mathcal{L}(\theta), \Pi).
    \end{equation}
\end{lemma}

Please be aware that the notation $ \mathcal{T}_p(\nabla\mathcal{L}(\theta), \Pi)$ signifies that we apply a permutation operation to the gradient of the loss function $\mathcal{L}$ with respect to the weight $\theta$.
\begin{proof}
Let $f_{\theta}$ be a neural network trained with a loss $\mathcal{L}(\theta)$ and $\Pi\in\mathbb{\Pi}$. We prove the case with $\Pi$ containing only Identity matrices except for the $i$-th matrix, that we denote $\vpi$. 

From property \ref{prop:backprop}, we have that $\forall l\in\llbracket1,L\rrbracket, \nabla_{W^{[l]}}\mathcal{L}(\theta) = \vdelta^{[l]}(\theta)\va^{[l-1]^\intercal}(\theta)$ with $\vdelta^{[l]}(\theta) = \nabla_{\vz^{[l]}}\mathcal{L}(\theta)$ and $\va^{[l-1]}(\theta)$ the activation after the layer $l-1$. We add the notation of dependence of $\vdelta^{[l]}$ and $\va^{[l]}$ in $\theta$ to express that they correspond to the gradient and the activation of the network of weights $\theta$. We can adapt this result to the layer $i$ and $\mathcal{T}(\theta, \Pi)$:
\begin{align}
\nabla_{\vpi W^{[i]}}\mathcal{L}(\mathcal{T}_p(\theta, \Pi)) & = \vdelta^{[i]}(\mathcal{T}_p(\theta, \Pi))\va^{[i-1]^\intercal}(\mathcal{T}_p(\theta, \Pi)) \\
& = \left(\left(W^{[i+1]} \vpi^\intercal\right)^\intercal\vdelta^{[i+1]}(\mathcal{T}_p(\theta, \Pi)) \circ \nabla_{\vpi\vz^{[l]}} \vpi\va^{[l]}(\theta) \right)\va^{[i-1]^\intercal}(\mathcal{T}_p(\theta, \Pi))\\
& = \left(\vpi W^{[i+1]^\intercal}\vdelta^{[i+1]}(T(\theta, \Pi))\circ \vpi\nabla_{\vz^{[l]}} \va^{[l]}(\theta)\right)\va^{[i-1]^\intercal}(\mathcal{T}_p(\theta, \Pi)) \\
& = \left(\vpi W^{[i+1]^\intercal}\vdelta^{[i+1]}(\theta)\circ\vpi\nabla_{\vz^{[l]}} \va^{[l]}(\theta)\right)\va^{[i-1]^\intercal}(\theta) \\
& = \vpi \vdelta^{[l]}(\theta)\va^{[i-1]^\intercal}(\theta) \\
&= \vpi \nabla_{W^{[i]}}\mathcal{L}(\theta)
\end{align}
This proof can be extended backward to any $\Pi\in\mathbb{\Pi}$ as the action of $\pi^\intercal$ and of $\pi$ cancel out for all the following weights.
\end{proof}

We proved Lemma \ref{lem:grad_equi} with an MLP, but this result extends to convolutional layers. To do this, we need to define that the permutation matrices apply on the channels of the weight tensors.

\subsection{Permutation-equivariance of stochastic gradient descent}

To simplify the notations, we use neural networks and weights indifferently as arguments for the optimization and symmetry operators $\star$ and $\mathcal{T}$. For instance, we will simplify $\star(f_{\theta})$ by $\star(\theta)$, implicitly assuming the architecture of $f_{\theta}$.

\begin{proposition}
\label{thm:app_commutativity}
    Let $f_{\theta_0}$ be a randomly initialized neural network and $\star_s$ the stochastic gradient descent operator applying $s$ times the update rule of learning rate $\gamma_s$. For all steps $s\in\mathbb{N}$, $\star_s$ is permutation equivariant. In other words,
    \begin{equation}
        \forall s\in\mathbb{N}, \forall\Pi\in\mathbb{\Pi}, \ \star_s(\mathcal{T}_p(\theta, \Pi)) = \mathcal{T}_p(\star_s(\theta), \Pi).
    \end{equation}
\end{proposition}

We provide the proof below for multi-layer perceptrons, but it extends to convolutional neural networks.

\begin{proof}
Let $\Pi\in\mathbb{\Pi}$ be a permutation set applied on a neural network $f_\theta$ trained with an objective $\mathcal{L}$ of gradient $\nabla\mathcal{L}$. Let us define $P(s)$ as follows: 

\begin{equation}
    P(s) = "\mathcal{T}_p \left(\star_s \left(\theta_0\right), \Pi \right) = \star_s \left(\mathcal{T}_p \left(\theta_0, \Pi \right)\right)" %
\end{equation}

We give a proof that $P(s)$ is verified for all $s\in\mathbb{N}$ by induction on $s$.

\emph{Base case:} $P(0)$ is trivially true as $\star_0 \left(f_{\theta_0}\right) = f_{\theta_0}$, hence 
\begin{equation}
\star_0 \left(\mathcal{T}_p \left(\theta_0, \Pi \right)\right) = \mathcal{T}_p \left(\theta_0, \Pi \right) = \mathcal{T}_p \left(\star_0 \left(\theta_0\right), \Pi \right).
\end{equation}

\emph{Induction step:} Assume the induction hypothesis that for a particular $k$, the single case $s = k$ holds, meaning $P(k)$ is true:
\begin{equation}
    \mathcal{T}_p \left(\star_k \left(\theta_0\right), \Pi \right) = \star_k \left(\mathcal{T}_p \left(\theta_0, \Pi \right)\right)
\end{equation}

Let us start with $\mathcal{T}_p(\star_{k+1}(\theta_{0}), \Pi)$. We have that 
\begin{align}
    \mathcal{T}_p(\star_{k+1}(\theta_{0}), \Pi) = & \mathcal{T}_p(\theta_{k+1}, \Pi) \\
    = & \mathcal{T}_p(\theta_k - \gamma_s\nabla\mathcal{L}(\theta_k), \Pi) \\
    = & \mathcal{T}_p(\star_k(\theta_0), \Pi) - \gamma_s\mathcal{T}_p(\nabla\mathcal{L}(\star_k(\theta_0)), \Pi) \\
    = & \star_k(\mathcal{T}_p(\theta_0), \Pi) - \gamma_s\nabla\mathcal{L}(\star_k(\mathcal{T}_p(\theta_0, \Pi)))  \label{proofeq:equi} \\
    = & \star_{k+1}(\mathcal{T}_p(\theta_0, \Pi)), \label{proofeq:end}
\end{align}

Equation \ref{proofeq:equi} uses the permutation-equivariance of both the gradient (see Lemma~\ref{lem:grad_equi}) and the training operator until step $k$. Furthermore, equation~\ref{proofeq:end} is exactly $P(k+1)$, establishing the induction step.

\emph{Conclusion:} Since both the base case and the induction step have been proven, by induction, the statement $P(s)$ holds for every natural number $s$.
\end{proof}

We provide the proof of Proposition \ref{thm:app_commutativity} with SGD for simplicity. However, it intuitively extends to SGD with momentum since the momentum is also permutation-equivariant and updated in a permutation-equivariant way. Furthermore, it also generalizes to (mini-)batch stochastic gradient descent, provided that the batches contain the same images. %

While this also extends to the Adam optimizer~\citep{kingma2014adam}, this is most likely not true for new meta-learned optimizers such as VeLO~\citep{metz2022velo}.

\subsection{Equiprobability of the permutations}

We build upon the results from the previous sections to propose the following proposition.

\begin{proposition}
\label{th:permutations_equiprobable}
Let $f_\theta$ be a neural network with initial weights independently and layer-wise identically distributed. The probability of converging towards any symmetrically permuted network $\mathcal{T}_p(f_\theta, \Pi)$ given the dataset and training hyperparameters does not depend on the permutation set $\Pi$. In other words,
\begin{equation}
\label{eq:permutation_equiprobable}
       p(\Theta_s=\mathcal{T}_p(\theta_s, \Pi) | \mathcal{D}) = p(\Theta_s=\theta_s | \mathcal{D}).
\end{equation}
\end{proposition}
\begin{proof}
For simplicity, we assume an optimization scheme with early stopping and final weights $\theta_s$ at the end of step $s$. Given $\theta_s$, we can denote $\theta_{0\rightharpoonup s}$ the space of initializations with non-zero probability to converge to $\theta_s$ after $s$ steps.

Let there be a permutation set $\Pi$ leaving $f_{\theta_s}$ invariant. We want to prove \Eqref{eq:permutation_equiprobable}, i.e., that the probability to converge to $\mathcal{T}_p(\Theta_s=\theta_s, \Pi)$ given $\mathcal{D}$ is the same as for $\theta_s$ for any permutation set $\Pi\in\mathbb{\Pi}$.

We can inject information about the potential initialization points, defining $\theta_{0\rightharpoonup s} = \{\theta | \star_s[\theta] = \theta_s\}$ as the pre-image of $\theta_s$ by the optimization procedure with $s$ steps. It comes that
\begin{align}
\label{eqproof:cond}
    p(\Theta_s =\theta_s,\mathcal{D})=p(\Theta_0 \in \theta_{0\rightharpoonup s}, \Theta_s = \theta_s,\mathcal{D}). 
\end{align}
Indeed, as we use the notation $\Theta_0 \in \theta_{0\rightharpoonup s}$ for
\begin{equation}
    p(\Theta_0 \in \theta_{0\rightharpoonup s}) = \hspace{-1em} \int\limits_{\theta \in \theta_{0\rightharpoonup s}} \hspace{-1em} p(\Theta_0 =\theta)d\theta,
\end{equation}
we can get, with $\Omega$ the set of possible weights,
\begin{align}
    \int\limits_{\theta \in \Omega} \hspace{-0.4em} p(\Theta_0 =\theta | \Theta_s = \theta_s,\mathcal{D})d\theta  = \hspace{-1em}\int\limits_{\theta \in \theta_{0\rightharpoonup s}} \hspace{-1em} p(\Theta_0 =\theta | \Theta_s = \theta_s,\mathcal{D})d\theta  +  \hspace{-1.4em}  \int\limits_{\theta \in \Omega \setminus \theta_{0\rightharpoonup s}} \hspace{-1.4em} p(\Theta_0 =\theta | \Theta_s = \theta_s,\mathcal{D})d\theta \\
\end{align}
The left term is equal to 1, and so is the middle term as, by definition of $\theta_{0\rightharpoonup s}$, it is impossible to have initialized the weights in $\Omega \setminus \theta_{0\rightharpoonup s}$ and converge to $\theta_s$ given the dataset $\mathcal{D}$. Since $p(\Theta_0 \in \theta_{0\rightharpoonup s}, \Theta_s = \theta_s,\mathcal{D}) = p(\Theta_0 \in \theta_{0\rightharpoonup s}|\Theta_s = \theta_s,\mathcal{D})p(\Theta_s = \theta_s,\mathcal{D})$, we have \Eqref{eqproof:cond}.

Furthermore, the definition of conditional probabilities implies that the probability of converging to $\theta_s$ given $\mathcal{D}$ can be rewritten as
\begin{align}
    p(\Theta_s=\theta_s | \mathcal{D}) &= \frac{p(\Theta_0\in\theta_{0\rightharpoonup s}, \Theta_s=\theta_s, \mathcal{D})\cdot p(\Theta_0 \in \theta_{0\rightharpoonup s})}{p(\mathcal{D})\cdot p(\Theta_0 \in \theta_{0\rightharpoonup s})} \\
    &= \frac{p(\Theta_0 \in \theta_{0\rightharpoonup s}, \Theta_s=\theta_s, \mathcal{D})\cdot p(\Theta_0 \in \theta_{0\rightharpoonup s})}{p(\Theta_0 \in \theta_{0\rightharpoonup s}, \mathcal{D})} \quad [\text{independence of} \ \Theta_0 \ \text{and} \ \mathcal{D}]\\
    &= p(\Theta_s=\theta_s | \Theta_0 \in \theta_{0\rightharpoonup s},\mathcal{D}) p(\Theta_0 \in \theta_{0\rightharpoonup s}).
    \label{eq:trained_init}
\end{align}

We can also apply \Eqref{eq:trained_init} for a permuted network. Let there be a permutation set $\Pi$; \Eqref{eq:trained_init} is also valid for the permuted network $\mathcal{T}_p(\theta_s, \Pi)$ with its pre-image the set $\mathcal{T}_p(\theta_{0\rightharpoonup s}, \Pi)$ (defined as $\{\mathcal{T}_p(\theta, \Pi)\}_{\theta \in \theta_{0\rightharpoonup s}}$). Indeed, Proposition~\ref{thm:app_commutativity} exactly proves that training a network starting from the permuted weights $\mathcal{T}_p(\theta_0, \Pi)$ converges towards the permutation of the original weights  $\mathcal{T}_p(\theta_s, \Pi)$.
\begin{equation}
    \label{eq:core}
    p(\Theta_s=\mathcal{T}_p(\theta_s, \Pi) | \mathcal{D}) = p(\Theta_s=\mathcal{T}_p(\theta_s, \Pi) | \Theta_0\in\mathcal{T}_p(\theta_{0\rightharpoonup s}, \Pi),\mathcal{D}) p(\Theta_0\in\mathcal{T}_p(\theta_{0\rightharpoonup s}, \Pi))
\end{equation}

\emph{Deterministic Training:} Suppose the optimization algorithm is fully deterministic, given the initialization state. In practice, the dataloader seed would be set to some value, and the backpropagation kernels would be chosen as deterministic. In that case, the convergence to $\theta_s$ is fully determined by the initialization $\Theta_0$ and the dataset $\mathcal{D}$, i.e., 
\begin{equation}
    \label{eq:cond_deterministic}
    p(\Theta_s=\theta_s | \Theta_0 \in \theta_{0\rightharpoonup s},\mathcal{D}) =1
\end{equation}

\emph{Non-deterministic Training:} In the non-deterministic setting,  corresponding in practice to the order of the inputs in the batches as well as the stochasticity of backpropagation algorithms. We can marginalize the (eventually multivariate) sources of stochasticity $\xi\in \Xi$ to regain a deterministic relationship.

\begin{align}
\label{eq:cond_nondeterministic}
p(\Theta_s=\theta_s| \Theta_0\in\theta_{0\rightharpoonup s},\mathcal{D})&=\int\limits_{\xi\in\Xi} p(\Theta_s=\theta_s | \Theta_0 \in \theta_{0\rightharpoonup s}, \xi, \mathcal{D}) d\xi \\
&=\int\limits_{\xi\in\Xi} p(\Theta_s=\mathcal{T}_p(\theta_s, \Pi) | \Theta_0\in\mathcal{T}_p(\theta_{0\rightharpoonup s}, \Pi), \xi, \mathcal{D})d\xi\\
&= p(\Theta_s=\mathcal{T}_p(\theta_s, \Pi) | \Theta_0\in\mathcal{T}_p(\theta_{0\rightharpoonup s}, \Pi),\mathcal{D})
\end{align}

In the case of deterministic training, the left term of \Eqref{eq:core} is equal to 1 from \Eqref{eq:cond_deterministic} with the permuted weights $\mathcal{T}_p(\theta_0, \Pi)$ and $\mathcal{T}_p(\theta_s, \Pi)$. 

For the right term, $p(\Theta_0\in\mathcal{T}_p(\theta_{0\rightharpoonup s}, \Pi))=p(\Theta_0\in\theta_{0\rightharpoonup s})$, since the initial weights are layer-wise identically distributed and the permutations act separately inside the layers.

Finally, we have
\begin{align}
    p(\Theta_s=\mathcal{T}_p(\theta_s, \Pi) | \mathcal{D}) = p(\Theta_s=\theta_s | \mathcal{D}), 
\end{align}

which is Proposition \ref{th:permutations_equiprobable}.
\end{proof}

\subsection{Theoretical insights on the Bayesian posterior}

We start with the following proposition, independent from Proposition~\ref{prop:mixture} but that we still deem interesting for the reader.

\begin{proposition}
    Let $f_\vomega$ be a neural network, $\vx$ an input vector, and $y$ a scalar. We also denote $\Omega_>$ the space of sorted weights, and $\mathbb{\Pi}$ remains the set of possible permutations for $f_\vomega$. We have
    \begin{equation}
         p(y \mid \vx, \mathcal{D})=|\mathbb{\Pi}|\int\limits_{\vomega\in\Omega_>} p(y \mid \vx,\vomega) p(\vomega \mid \mathcal{D}) d\vomega.
    \end{equation}
\end{proposition}

\begin{proof}
    Let $f_\vomega$ be a neural network. We recall the marginalization on the network's weights:
    \begin{equation}
         p(y \mid \vx, \mathcal{D})=\int\limits_{\vomega\in\Omega} p(y \mid \vx, \vomega) p(\vomega \mid \mathcal{D}) d\vomega.
    \end{equation}
    We can split the space $\Omega$ and get $\Omega = \Omega_{\geq} \cup \hspace{-0.5em} \bigcup\limits_{i\in\llbracket1,|\mathbb{\Pi}|\rrbracket}\hspace{-0.5em}\Omega_{>_i}$, with $\Omega_{>_i}$ the space of the weights ordered and permuted according to $\Pi$ and $\Omega_{\geq}$ the space including corner-cases: the networks with at least twice the same weights in a layer. In our case, $\Omega_{\geq}$ is theoretically of measure zero. In practice, we note that the reals are discretized on a computer, and the edge cases may happen with an extremely small probability, decreasing with the chosen precision. Denote $\Omega_>$ the space with descending ordered weights, we have
    \begin{equation}
         p(y \mid \vx, \mathcal{D})=\sum\limits_{i=1}^{|\mathbb{\Pi}|}\int\limits_{\vomega\in\Omega_>} p(y \mid \vx, \mathcal{T}_p(\vomega, \Pi_i)) p(\mathcal{T}_p(\vomega, \Pi_i) \mid \mathcal{D}) d\vomega.
    \end{equation}
    By definition of the symmetry operator, we know that it does not affect the prediction of the transformed network:
    \begin{equation}
        \forall i\in\llbracket1,|\mathbb{\Pi}|\rrbracket, \ p(y \mid \vx, \mathcal{T}_p(\vomega, \Pi_i))= p(y \mid \vx, \vomega).
    \end{equation}
    Moreover, Proposition \ref{th:permutations_equiprobable} provides 
    \begin{equation}
        \forall i,j\in\llbracket1,|\mathbb{\Pi}|\rrbracket, \ p(\mathcal{T}_p(\vomega, \Pi_i)\mid \mathcal{D})=p(\vomega\mid \mathcal{D}).
    \end{equation}
    It follows that
    \begin{equation}
         p(y \mid \vx, \mathcal{D})=|\mathbb{\Pi}|\int\limits_{\vomega\in\Omega_>} p(y \mid \vx,\vomega) p(\vomega \mid \mathcal{D}) d\vomega.
    \end{equation}
\end{proof}

Then, we recall Proposition~\ref{prop:mixture} of the main paper:

\begin{proposition}
Define $f_\omega$ a neural network and $\tilde{\omega}$ its corresponding identifiable model - a network transformed for having sorted unit-normed neurons. Let us also denote $\mathbb{\Pi}$ and $\mathbb{\Lambda}$, respectively, the sets of permutation sets and scaling sets, and $\tilde{\Omega}$ the random variable of the sorted weights with unit norm. The Bayesian posterior of a neural network $f_\omega$ trained with stochastic gradient descent can be expressed as a continuous mixture of a discrete mixture:
 \begin{equation}
         p(\Omega=\vomega \mid \mathcal{D}) = |\mathbb{\Pi}|^{-1} \hspace{-0.3em}
     \int\limits_{\Lambda \in \mathbb{\Lambda}} \hspace{-0.3em} \sum\limits_{\Pi\in\mathbb{\Pi}}p(\tilde{\Omega} = \mathcal{T}_p(\mathcal{T}_s(\vomega,\Lambda), \Pi), \Lambda \mid \mathcal{D})d\Lambda.
\end{equation}
\end{proposition}

\begin{proof}
Let $f_\omega$ be a neural network and start with the right term. We can denote by $\tilde{\Pi}$ and $\tilde{\Lambda}$ the only couple of permutations and scale such that the weights $\mathcal{T}_p(\mathcal{T}_s(\vomega,\tilde{\Lambda}), \tilde{\Pi})$ are sorted and unit-normed. Given that $\tilde{\Omega}$ is also sorted and unit-normed, $\tilde{\Pi}$ and $\tilde{\Lambda}$ are the only supports of the respective sum and integral. Therefore, we have that  %
\begin{equation}
    |\mathbb{\Pi}|^{-1}
     \int\limits_{\Lambda \in \mathbb{\Lambda}} \sum\limits_{\Pi\in\mathbb{\Pi}}p(\tilde{\Omega} = \mathcal{T}_p(\mathcal{T}_s(\vomega, \Lambda), \Pi), \Lambda \mid \mathcal{D})d\Lambda = |\mathbb{\Pi}|^{-1}p(\tilde{\Omega} = \mathcal{T}_p(\mathcal{T}_s(\vomega,\tilde{\Lambda}), \tilde{\Pi}), \tilde{\Lambda}\mid \mathcal{D}).
\end{equation}

In parallel, we have that, with $\tilde{\vomega}$ the sorted unit-normed weights associated to $\vomega$, i.e., $\tilde{\vomega}=\mathcal{T}_p(\mathcal{T}_s(\vomega,\tilde{\Lambda}), \tilde{\Pi})$, 
\begin{align}
    p(\Omega=\vomega|\mathcal{D}) &= p(\tilde{\Omega}=\tilde{\vomega}, \tilde{\Pi}, \tilde{\Lambda}|\mathcal{D}) \\
    &= p(\tilde{\Omega}=\tilde{\vomega}, \tilde{\Lambda}|\mathcal{D}) p(\tilde{\Pi}) \\
    &= |\mathbb{\Pi}|^{-1}p(\tilde{\Omega}=\tilde{\vomega}, \tilde{\Lambda}|\mathcal{D}).
\end{align}
Indeed, $\Pi$ and $\mathcal{D}$ are independent and $p(\Pi)$ is determined by the initialization of the neural network that we suppose layer-wise identically and independently distributed. We obtain the Proposition directly when replacing $\tilde{\vomega}$ by $\tilde{\vomega}=\mathcal{T}_p(\mathcal{T}_s(\vomega,\tilde{\Lambda}), \tilde{\Pi})$.
\end{proof}

\subsection{The minimum scaled network mass problem}
\label{app:loglogconvex}

In this section, we develop the formalism leading to the proof of the log-log strict convexity of the minimum scaled network mass problem (or \emph{min-mass} problem). We start by recalling its definition.

\subsubsection{Definitions}

\begin{definition}[Minimum scaled network mass problem]
\label{def:scaled_problem}
    Let $f_{\theta}$ a neural network and $\bar{\theta}$ its weights without the biases. We define the minimum scaled network mass problem or \emph{min-mass} problem as the minimization of $f_{\theta}$'s L2 regularization term (the "mass") under invariant scaling transformations. In other words,
    \begin{equation}
        m^* = \min\limits_{\Lambda\in\mathbb{\Lambda}}\ \left|\mathcal{T}(\bar{\theta}, \Lambda)\right|^2.
    \end{equation}
    We also denote the mass of the neural network $f_{\theta}$ as $m(f_{\theta}) = \left|\mathcal{T}(\bar{\theta}, \Lambda)\right|^2$.
\end{definition}

In this work, we are especially interested in the versions of this problem that are not degenerate. This condition was not given in the main paper for clarity, but please note that it is by no means restrictive for real-world cases and could be handled properly for theory with some caution.

\begin{definition}
    We say that the \emph{min-mass} problem of the network neural network $f_{\theta}$ is non-degenerate if, taking any $\Lambda\in\mathbb{\Lambda}$, we have that for all scales $\lambda\in\Lambda$, at least one term in $\lambda_i^{-1}$ is non-zero.
\end{definition}

The fact that the problem is non-degenerate is important as it could lead to the non-desirable divergence of some parameters in the problem.

\subsubsection{Computing the mass of the network}

For a neural network without skip-connections and constituted of linear and convolutional layers only, we provide first a definition and use it to devise an important property:

\begin{definition}
     Let $W$ be the weight of a linear or convolutional layer. We denote $M(W)$, the matrix representation of the mass of the network. For linear layers, we have, with $\circ$, the Hadamard product,
     \begin{equation}
         M(W) = W\circ W,
     \end{equation}
     and for convolutional layers of kernel size $(h, v)$,
     \begin{equation}
         M(W) = \sum\limits_{i=1}^{h}\sum\limits_{j=1}^{v}\left(W\circ W\right)[:, :, i, j].
     \end{equation}
\end{definition}

\begin{property}
\label{prop:expr}
     Let $f_{\theta}$ be an $L$-layer neural network of weights $\theta=\{W^{[l]}, b^{[l]}\}_{l\in\llbracket1,L\rrbracket}$. The mass of the neural network scaled by $\Lambda\in\mathbb{\Lambda}$ with $\Lambda=\{\vlambda_l\}_{l\in\llbracket1,L-1\rrbracket}$ and $\vlambda_0=\vlambda_L=\vone$ can be expressed as, with $\mathrm{sum} (W)$ the sum of all elements of W,
     \begin{align}
         m\left(f_{\mathcal{T}_s({\theta}, \Lambda)}\right) = \sum\limits_{l=1}^{L} \mathrm{sum} \left(\vlambda_l^2\bigtriangledown\left(\vlambda_{l-1}^{-2}\rhd M\left(W^{[l]}\right)\right)\right)
     \end{align}
\end{property}

This result can be extended to networks with batch normalization layers and residual connections. To build the convex problem out of this formula, we use the connection between the line-wise and column-wise products and linear algebra, as presented in \ref{par:equiv_linear_algebra}. 

\subsubsection{Log-log convexity}

Finally, we have the following core property on \emph{min-mass} problems:

\begin{proposition}
     Let $f_{\theta}$ be a neural network. The corresponding minimum scaled network-mass problem is log-log strictly convex, provided that it is non-degenerate. As such, the problem is equivalent to a strictly convex problem on $\mathbb{R}^{|\mathbb{\Lambda}|}$ and admits a global minimum attained in a single point denoted $\Lambda^*$.
\end{proposition}

\begin{proof}
    Let $f_{\theta}$ be an $L$-layer neural network of weights $\theta=\{W^{[l]}, b^{[l]}\}_{l\in\llbracket1,L\rrbracket}$. We recall the corresponding \emph{min-mass} problem:
\begin{equation}
    P \equiv \min\limits_{\Lambda\in\mathbb{\Lambda}}\ m\left(f_{\mathcal{T}_s({\theta}, \Lambda)}\right).
\end{equation}
    With Property \ref{prop:expr}, we have, for $\Lambda\in\mathbb{\Lambda}$ with $\Lambda=\{\vlambda_l\}_{l\in\llbracket1,L-1\rrbracket}$ and $\vlambda_0=\vlambda_L=\vone$, 
\begin{equation}
    P \equiv \min\limits_{\Lambda\in\mathbb{\Lambda}} \sum\limits_{l=1}^{L} \mathrm{sum} \left(\vlambda_1^2 \bigtriangledown \left(\vlambda_{l-1}^2 \rhd M\left(W^{[l]}\right)\right)\right).
\end{equation}
Let us now write $\vu_l = \log(\vlambda_{l})$ for all $l\in\llbracket1,L-1\rrbracket$. This is possible because due to the non-negative homogeneity of the ReLU activation, the $\vlambda_{l}$ were already chosen with elements in $\mathbb{R}_{>0}$. Please note that the need for this change of variables is why $P$ will be referred to as log-log strictly convex. We have,
\begin{equation}
    P \equiv \min\limits_{\Lambda\in\mathbb{\Lambda}} \sum\limits_{l=1}^{L} \mathrm{sum} \left(e^{2\vu_l} \bigtriangledown \left(e^{2\vu_{l-1}} \rhd M\left(W^{[l]}\right)\right)\right)
\end{equation}
We can convert our operators to products of diagonal matrices and Hadamard products using the formulae from \ref{par:equiv_linear_algebra}. It comes that,
\begin{align}
    P \equiv \min\limits_{\Lambda\in\mathbb{\Lambda}} \Big[ & \mathrm{sum}\left(\mathrm{diag} \left(e^{2\vu_1}\right)M\left(W^{[1]}\right)\right) \\
    & + \sum\limits_{l=2}^{L-1} \mathrm{sum} \left(\left(e^{2\vu_l^\intercal} e^{2\vu_{l-1}^{-1}}\right) \circ M\left(W^{[l]}\right)\right) \\
    & + \mathrm{sum} \left(M\left(W^{[L]}\right) \mathrm{diag} \left(e^{-2\vu_{L-1}}\right) \right) \Big]. \label{eqproof:convex}
\end{align}
In \Eqref{eqproof:convex}, we see that P is equivalent to a sum of convex and strictly convex functions. Indeed, the product $e^{2\vu_l^\intercal} e^{2\vu_{l-1}^{-1}}$ can be expressed as a matrix containing elements of the form $e^{2(\vu_l)_i-2(\vu_{l-1})_j}$. Therefore, $P$ is equivalent to a strictly convex problem. It comes that $P$ has, at most, a single solution. Moreover, since the elements of the $\vu_l$ are defined on $\mathbb{R}$ and the inner function is infinite at infinity (as a sum of positive terms with at least one of them diverging due as the problem is non-degenerate), we have that $P$ has a solution, which is unique.
\end{proof}

This proof can be easily extended to batch normalization layers that include one-dimensional weights. The extension to residual networks is a bit more complicated, as one should probably resort to graphs to write the equations properly.

We then prove the corollary proposed in the main paper.

\begin{corollary}
Let $f_\theta$ be a neural network trained with L2-regularization. Suppose the mass of $f_\theta$ is not optimal in the sense of the minimum scaled network mass problem. In that case, there is an infinite number of equivalent networks $\mathcal{T}_p(f_\theta, \Lambda)$ with training loss lower than that of the original network.
\end{corollary}

\begin{proof}
    Let $f_\theta$ be a neural network and denote the training loss by $\mathcal{L}(\theta) = \mathcal{L}_m(\theta) + \mathcal{L}_{\mathrm{L2}}(\theta)$ with $\mathcal{L}_{\mathrm{L2}}(\theta) = m(f_\theta)$ the L2-regularization term. We have, by definition of the symmetry operator $\mathcal{T}_s$,
    \begin{equation}
        \forall\Lambda\in\mathbb{\Lambda},\ \mathcal{L}_m(\theta) = \mathcal{L}_m(\mathcal{T}_s(\theta, \Lambda))
    \end{equation}
    Take $\Lambda^*$ as the optimal scaling set, $\mathcal{L}_{\mathrm{L2}}(\theta)$ is continuous everywhere in $\mathbb{\Lambda}$ as a sum of products of continuous functions and hence continuous in $\Lambda^*\in\mathbb{\Lambda}$. Therefore, we can find an open ball $B$ around $\Lambda^*$, where $\forall\Lambda\in L,\ \mathcal{L}_m(\mathcal{T}_s(\theta, \Lambda) > \mathcal{L}_m(\mathcal{T}_s(\theta, \Lambda^*)$. Hence the corollary.

\end{proof}

\section{Other discussions and details}
\label{sec:disc_details}
\subsection{Estimating the number of permutation-generated modes}
Let us suppose for simplicity that the layer-posterior of the identifiable model is an $n$-dimensional Dirac distribution of coordinates $\vtheta$ (or a distribution on a small bounded domain around $\vtheta$). The number of modes of the resulting posterior will be $(n-k)!$ when we define $k$ as the number of (nearly-)identical terms in $\vtheta$. To recall, $n$ is of the order of $10^{6000}$ for a ResNet-18 in good part due to the large final layers.

This result seems difficult to extend to other distributions. While there exist some results on the number of modes of a mixture of Gaussians, to the best of our knowledge, they either treat the case with two elements only~\citep{eisenberger1964genesis} or are conjectures of an upper bound~\citep{carreira2003number}.

On an experimental basis, we tried to use the MeanShift algorithm~\citep{comaniciu2002mean} to estimate the number of modes.
Despite the number of checkpoints, Mean Shift remained very sensitive to the bandwidth. It was, as expected, failing to detect modes in high dimensions but also providing unsatisfactory results in simpler cases, such as for the MLP.

\subsection{Details on the functional collapse in ensembles}
\label{sec:functional_collapse_details}
\begin{figure}[t]
    \centering
    \begin{subfigure}[b]{0.32\textwidth}
         \centering
         \includegraphics[width=\textwidth]{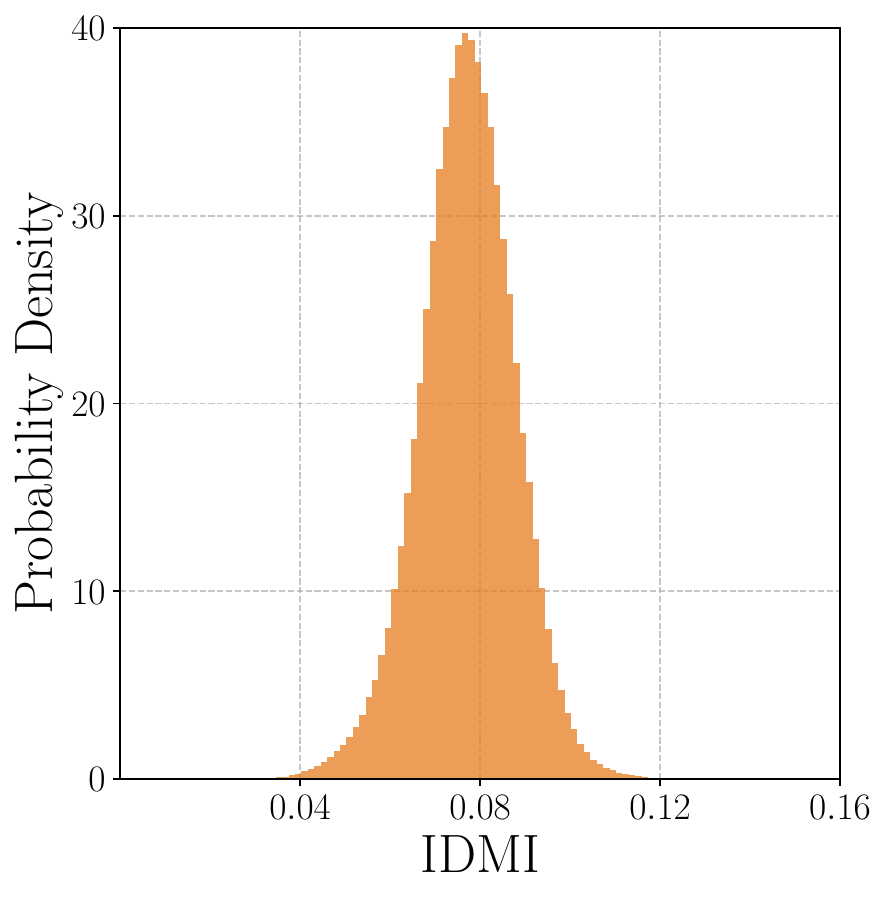}
     \end{subfigure}
     \hfill
     \begin{subfigure}[b]{0.32\textwidth}
         \centering
         \includegraphics[width=\textwidth]{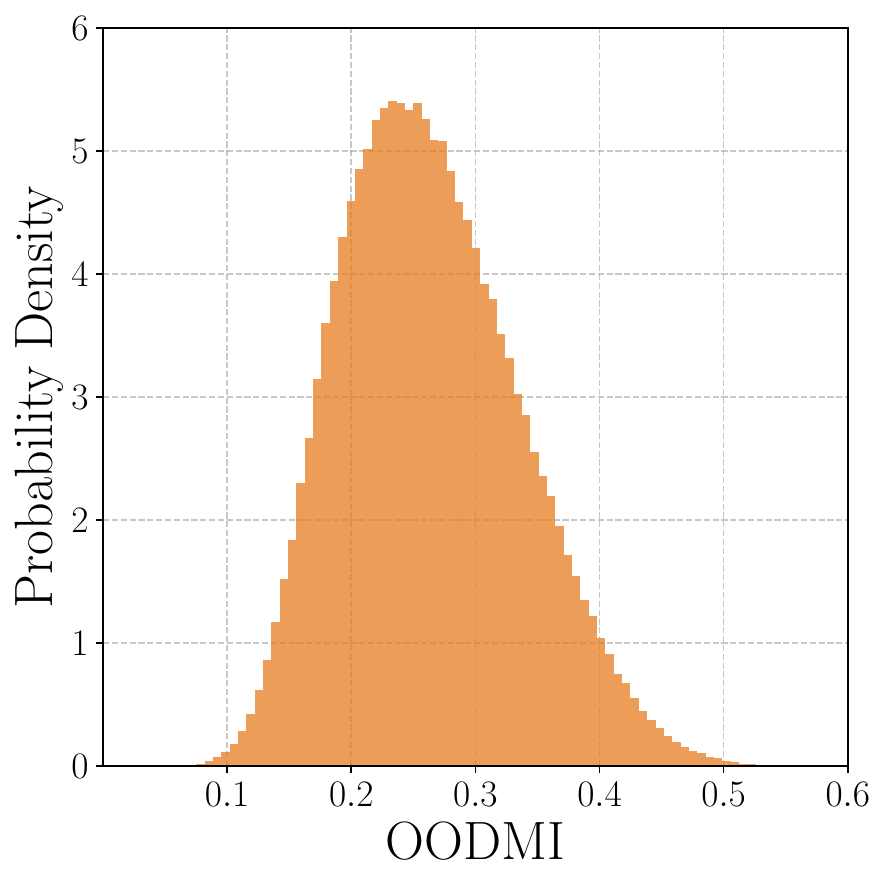}
     \end{subfigure}
     \hfill
     \begin{subfigure}[b]{0.32\textwidth}
         \centering
         \includegraphics[width=\textwidth]{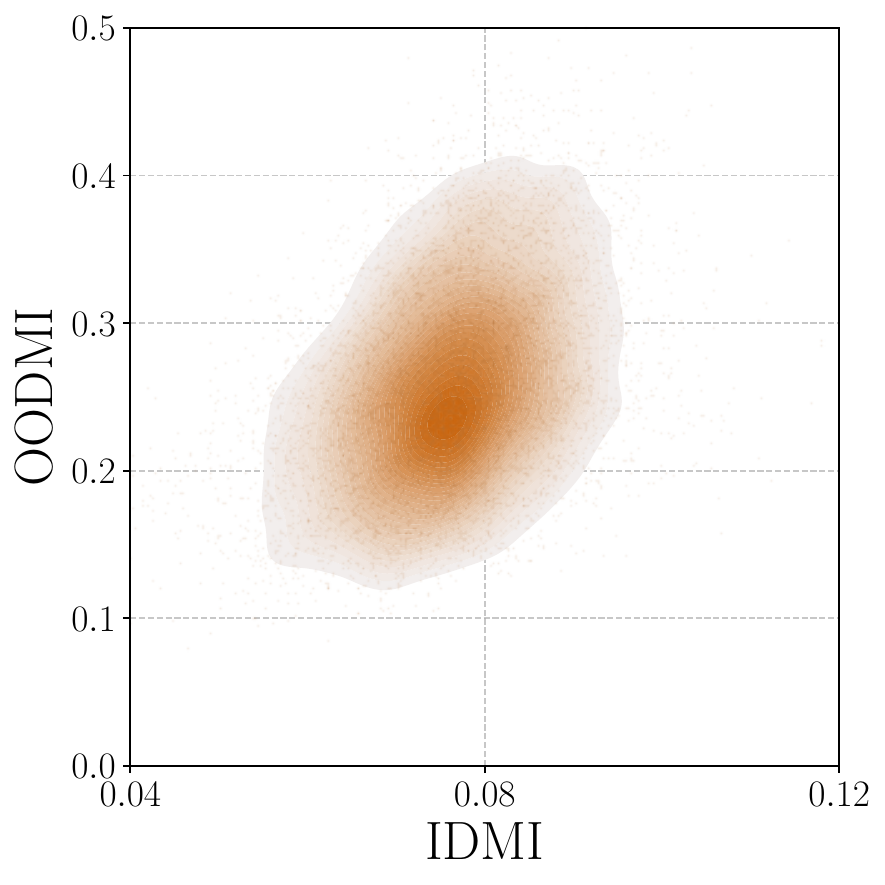}
     \end{subfigure}
        \caption{\textbf{There is a slightly more important functional collapse between couples of independently trained OptuNet on MNIST.} Contrary to ResNet-18 on CIFAR-100, the in-distribution and out-of-distribution mutual information (IDMI, resp. OODMI) seem correlated for small values.}
        \label{fig:functional_collapse_mnist}
\end{figure}

\paragraph{Functional collapse on MNIST/FashionMNIST with OptuNet}
In this paragraph, we study the potential functional collapse of OptuNet on MNIST and FashionMNIST, depicted in Figure \ref{fig:functional_collapse_mnist}. First, we see that the dispersion of the mutual information both in the training domain (right) and out of the domain (center) is greater than that of the ResNet-18. Indeed, in the left plot of Figure \ref{fig:functional_collapse_mnist}, there is a non-negligible probability of obtaining mutual information twice lower than the expected value. We note that the dispersion of the out-of-distribution values on FashionMNIST is also greater than for ResNets. Finally, in the case of OptuNet and contrasting with the results obtained with ResNets, there seems to be some positive correlation between the in-distribution MI and the OOD MI, at least when considering the lowest values. We can conclude that there is a high chance of \emph{functional collapse} in small networks, which is expected considering that their posterior is simpler. 

\paragraph{Functional collapse on TinyImageNet and Textures}
\begin{figure}[t]
    \centering
    \begin{subfigure}[b]{0.328\textwidth}
         \centering
         \includegraphics[width=\textwidth]{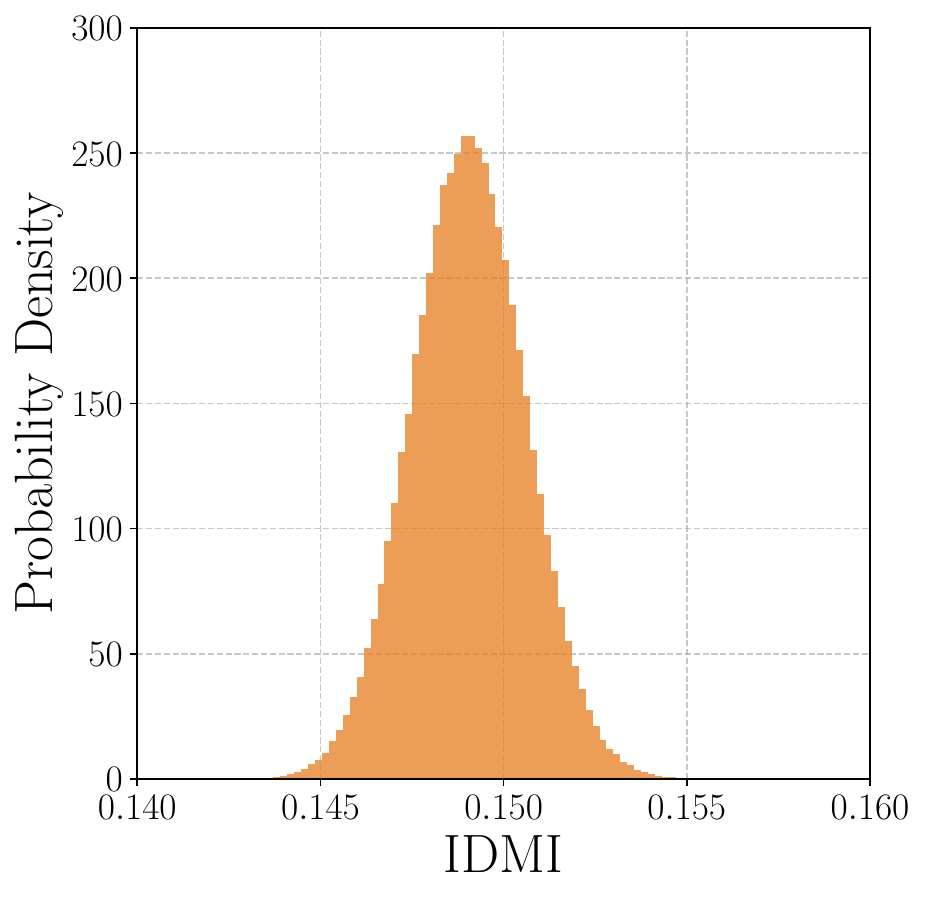}
     \end{subfigure}
     \hfill
     \begin{subfigure}[b]{0.318\textwidth}
         \centering
         \includegraphics[width=\textwidth]{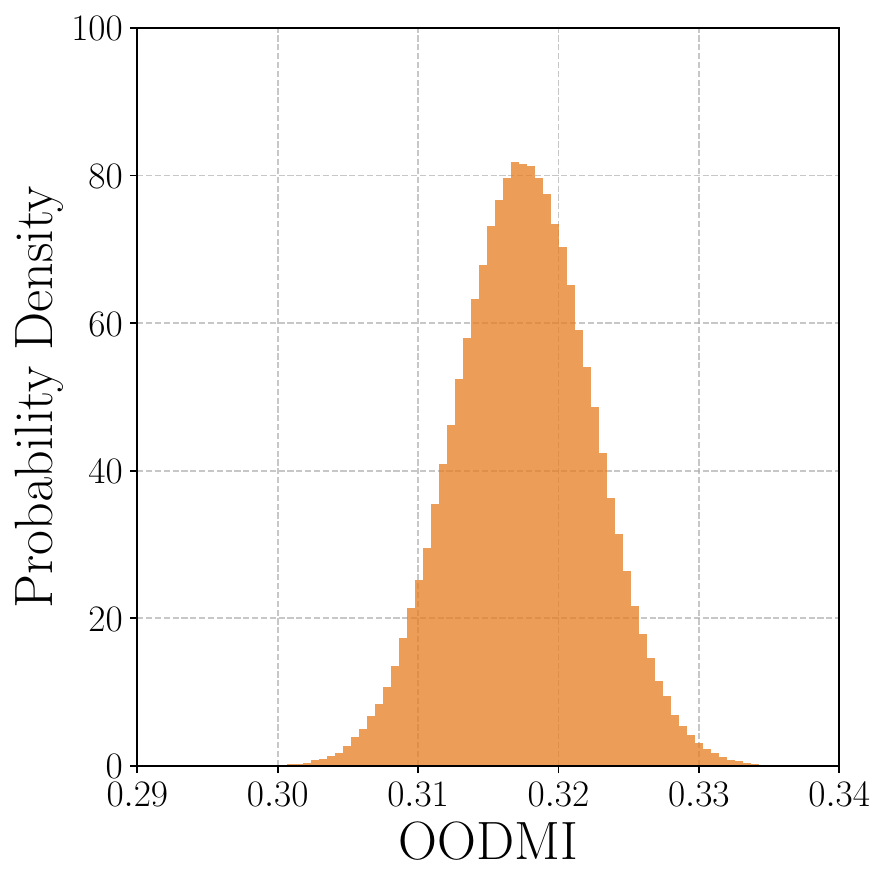}
     \end{subfigure}
     \hfill
     \begin{subfigure}[b]{0.314\textwidth}
         \centering
         \includegraphics[width=\textwidth]{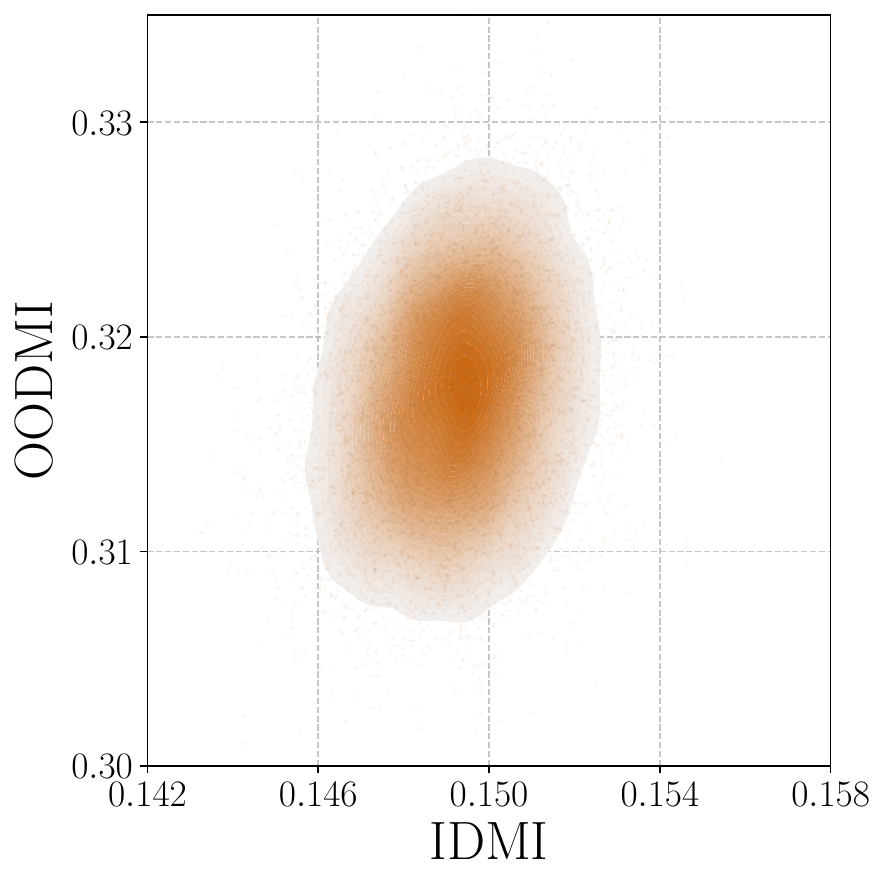}
     \end{subfigure}
        \caption{\textbf{The dispersion of the mutual information on TinyImageNet and Textures with ResNet-18 remains very small.} As for ResNet-18 on CIFAR-100, the in-distribution and out-of-distribution mutual information (IDMI, resp. OODMI) seem only very little correlated.}
        \label{fig:functional_collapse_tin}
\end{figure}

Figure \ref{fig:functional_collapse_tin} shows the mutual information on both the in-distribution (left) and out-of-distribution (center) datasets and their correlation (right). The plots are much closer to Figure \ref{fig:functional_collapse_c100}, which was expected as they share the same architecture. However, we see that the OOD dataset (here, Textures) potentially impacts the values of the mutual information. The dispersion (center) is lower than for ResNet-18 on SVHN, which may be more diverse. We measure the Pearson's $\rho$ correlation coefficient and obtain a statistically significant value of $\rho\approx0.06$. The two values are, therefore, positively correlated with a very low coefficient.

\subsection{Visualization}

\begin{figure}[t]
    \centering
    \begin{subfigure}[b]{0.7\textwidth}
         \centering
         \includegraphics[width=\textwidth]{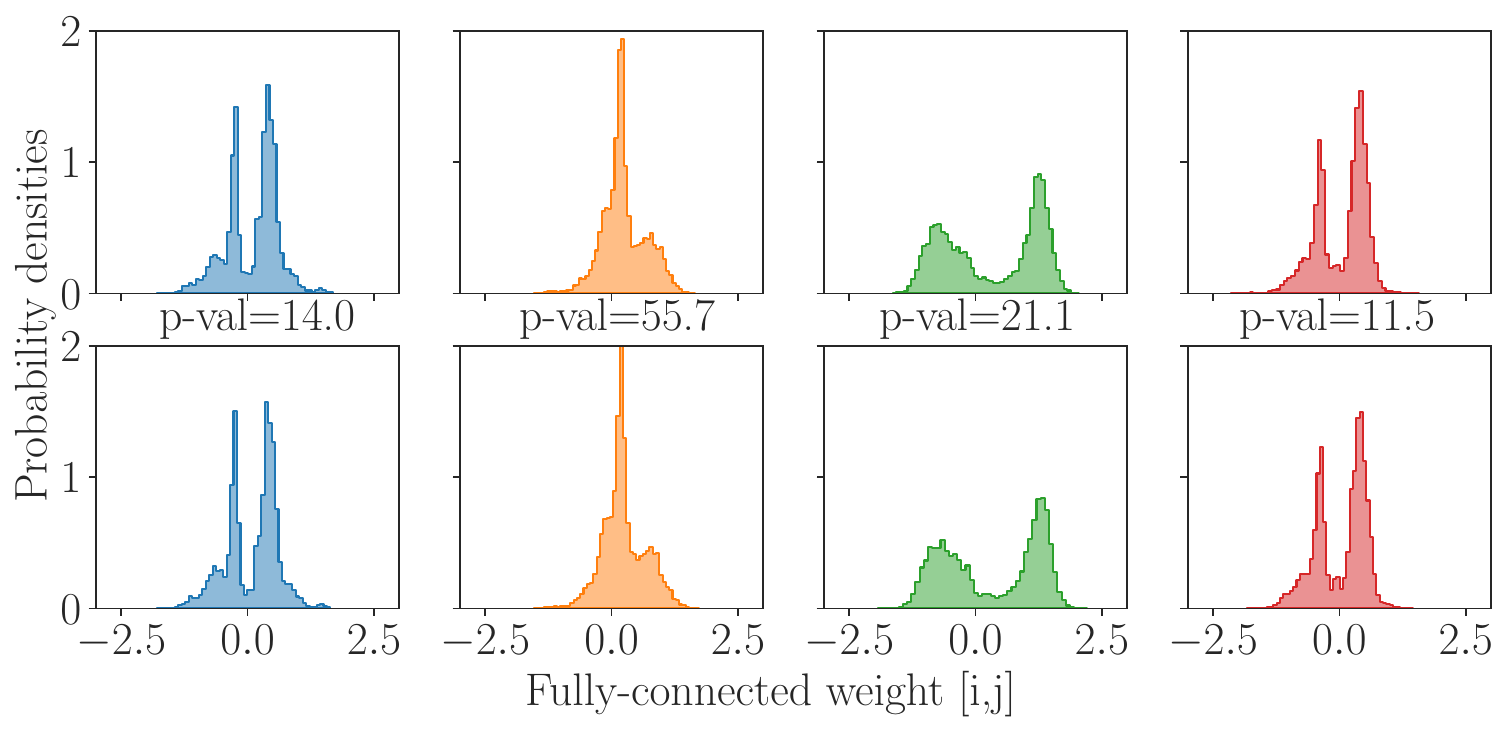}
     \end{subfigure}
     \hfill
     \begin{subfigure}[b]{0.7\textwidth}
         \centering
         \includegraphics[width=\textwidth]{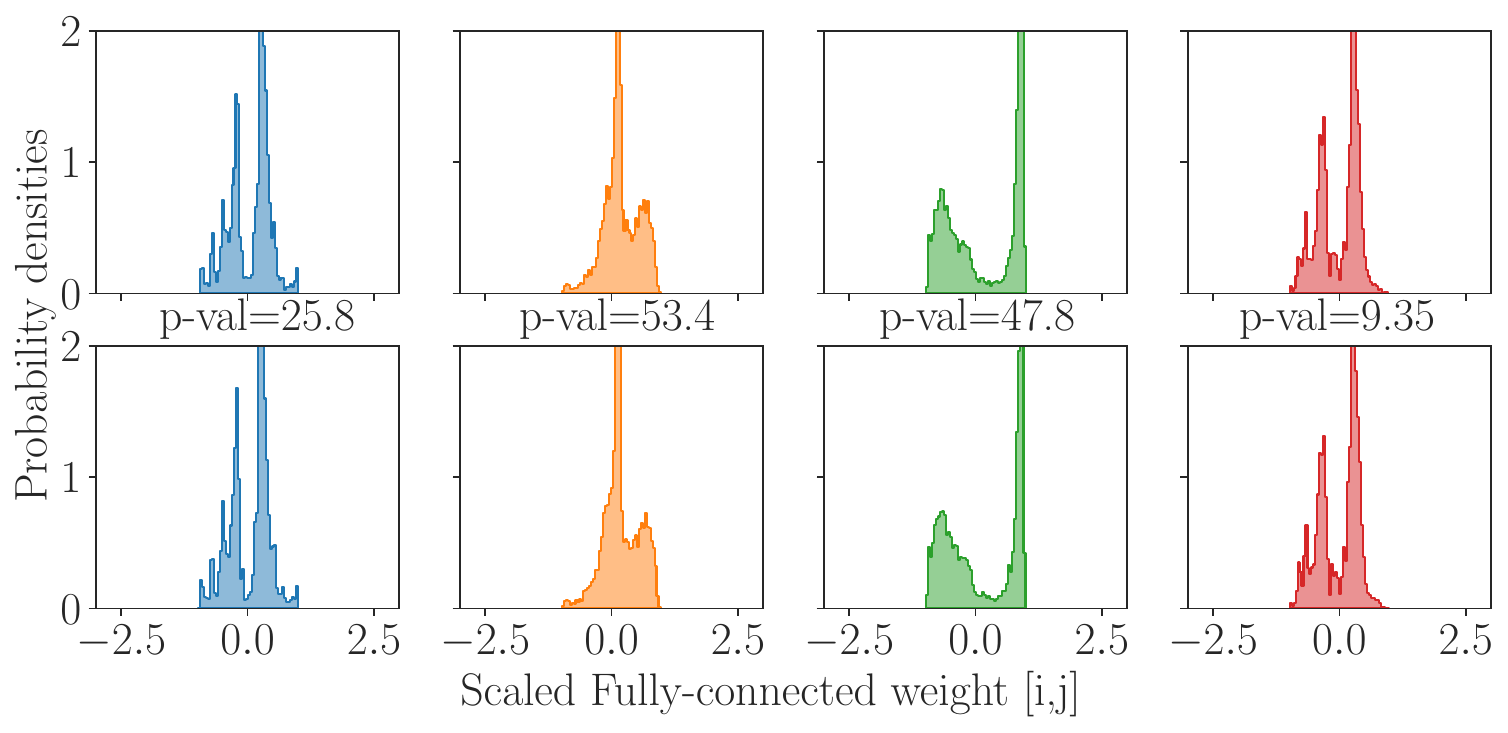}
     \end{subfigure}
     \hfill
     \begin{subfigure}[b]{0.7\textwidth}
         \centering
         \includegraphics[width=\textwidth]{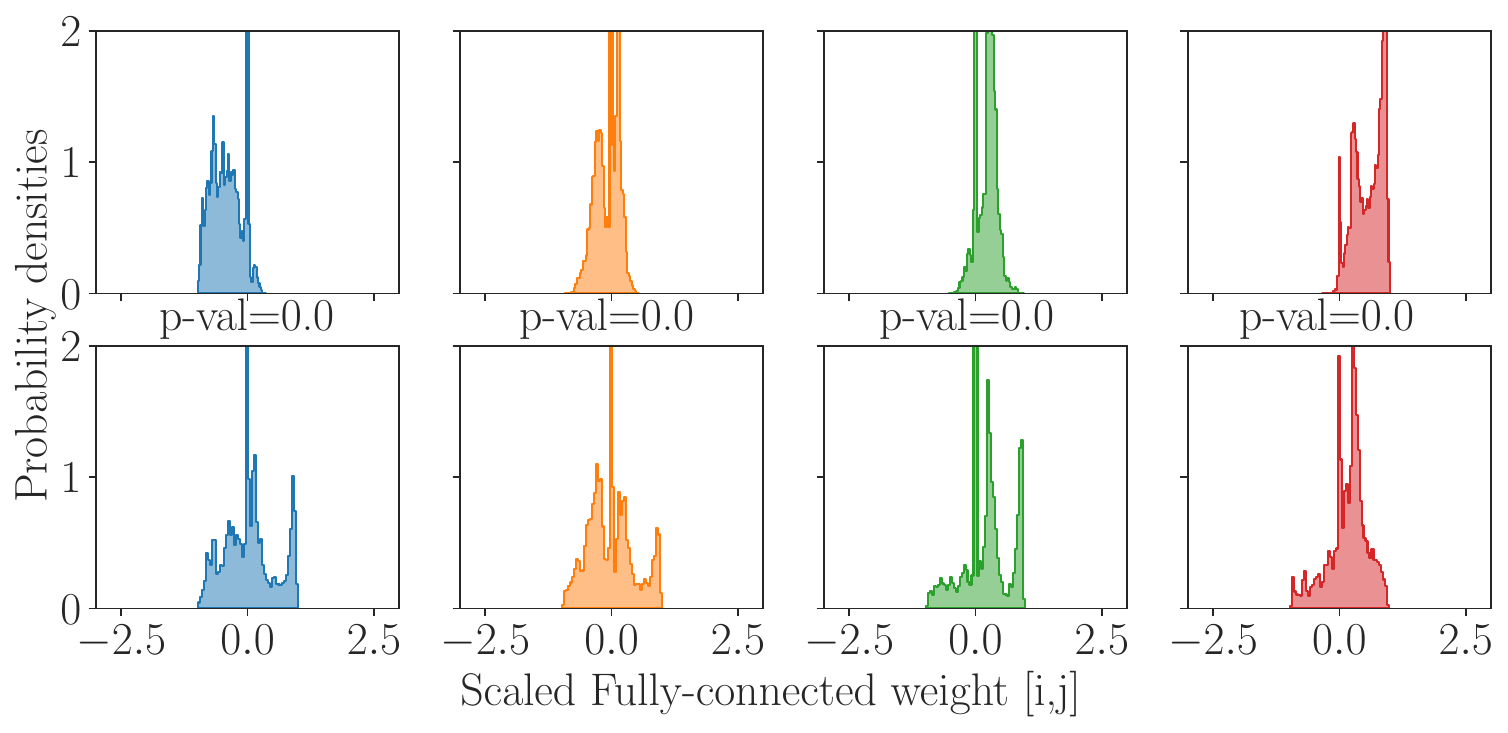}
     \end{subfigure}
        \caption{\textbf{Marginal posteriors of the last weights from a multi-layer perception:} The top rows are the original posterior, the middle rows have been scaled to a unit norm, and the bottom row has been permuted. We indicate the p-values of a Kolmogorov-Smirnov two-sample test between the marginals of the two neurons.}
        \label{fig:viz-mlp}
\end{figure}

In this section, we propose to visualize the marginals of the posterior of neural networks, including OptuNets. %
This visualization proposes the original marginal distributions and versions without scaling and permutation symmetries. As explained in the main paper, a corollary of Proposition~\ref{prop:mixture}, is that all layer-wise marginals are identical.

Figure~\ref{fig:viz-mlp} provides the marginal distributions of the weights of a multi-layer perceptron trained on energy-efficiency. We check with two-sample Kolmogorov-Smirnov tests that we cannot rule out that the distributions are the same, indicating the corresponding p-values. Indeed, as long as the permutations are not removed, we keep the equivalence property of all the marginals.

As stated in the abstract, the posterior lies in a very high dimension. We tried to apply dimension reduction with principal component analysis with scikit-learn~\citep{scikit-learn}, but the explained variance increased very slowly with the number of elements.

\subsection{Symmetries in modern layers}
\label{sec:generalization_leads}
In this section, we provide details on how we compute symmetries throughout the paper.

\subsubsection{Convolutional layers}

In Appendix~\ref{app:formalism}, we generalize the row-wise and column-wise products to enable their use between vectors and four-dimensional tensors. We also detail how to understand the product between permutation matrices and four-dimensional tensors. 

On a higher level, we see that the number of symmetries of convolutional layers is smaller compared to the number of parameters in a linear layer. Indeed, in the formalism we use, the number of permutations and scaling degrees of freedom (sDoF) of a linear layer correspond to its number of output neurons. In contrast, the number of permutations of a convolutional layer equals its number of output channels. The ratio of the number of permutations/sDoF over the number of parameters between linear and convolutional layers, therefore, equals the surface of the kernel of the convolutional layer.

\subsubsection{Grouped-Convolutional layers}

The number of permutations of a grouped convolutional layer~\citep{krizhevsky2012imagenet} reduces the number of permutations of the previous layer. Indeed, we have to account for the impossibility of fully reversing the permutation applied on the output neurons/channels of the preceding layer. This implies that the permutations of the previous layer need to remain intra-group. Say that the previous layer has $C_{\mathrm{out}}$ output channels and that the current layer has $\gamma$ groups. Instead of $C_{\mathrm{out}}!$ possible permutations, they are reduced to $\left(\frac{C_{\mathrm{out}}}{\gamma}!\right)^\gamma$. On the other hand, the number of output permutations does not change (provided that the next layer does not have groups). Likewise, the number of scaling degrees of freedom remains unchanged as scaling symmetries modify the values of the weights channel per channel.

\subsubsection{Batch normalization layers}

Counting permutation symmetries on batch normalization layers~\citep{ioffe2015batch} is quite straightforward. Regarding permutation symmetries, batch normalization layers are seamless, as one only needs to permute the weights, bias, running mean, and running variance. As such, they do not increase nor decrease the number of permutations. 

On the other hand, batch normalization layers can increase the number of degrees of freedom of scaling symmetries. Indeed, we can scale the running variances and means to the input scaling coefficients and then choose new output coefficients for the weights and biases. To conclude, they increase the number of sDoF by their number of features.

\subsubsection{Residual connections}

Residual connections tend to reduce the number of permutation and scaling symmetries~\citep{kurle2021symmetries}, and this is an important aspect in ResNets. Indeed, residual connections need the output coefficient of the last layers and the permutation matrices from all inner nodes in the computation graph to be equal. This potentially reduces both the number of possible permutations and the number of sDoF. In the case of ResNets, this effect is counterbalanced by the shortcut convolutional layers that can create permutations and sDoF, albeit less.

\end{document}